\documentclass{article}

\usepackage{microtype}
\usepackage{graphicx}
\usepackage{subcaption}
\usepackage{booktabs} 
\usepackage{tikz}
\usepackage{enumitem}
\usepackage{tcolorbox}
\usepackage{xcolor}
\usepackage{tikz}
\newcommand{\inner}[2]{\left\langle #1,\ #2 \right\rangle}
\usetikzlibrary{positioning, calc}
\definecolor{codeblue}{RGB}{230,241,251}
\definecolor{codetext}{RGB}{4,44,83}
\definecolor{theorylav}{RGB}{238,237,254}
\definecolor{theorytext}{RGB}{38,33,92}
\definecolor{boxpink}{RGB}{255,228,235}
\tcbuselibrary{breakable}

\newcommand{\Real}{\mathbb{R}}
\newcommand{\Complex}{\mathbb{C}}
\newcommand{\Integer}{\mathbb{Z}}
\usepackage{xcolor}
\usepackage{listings}
\usepackage{mathtools}

\DeclarePairedDelimiter{\norm}{\lVert}{\rVert}
\DeclarePairedDelimiter{\abs}{\lvert}{\rvert}
\newcommand{\PSD}{\succeq 0}
\newcommand{\Kstar}{K^{*}}
\DeclareMathOperator{\svec}{svec}
\DeclareMathOperator{\herm}{Herm}
\DeclareMathOperator{\tr}{tr}

\usetikzlibrary{arrows.meta, positioning, shapes.geometric, calc, fit, backgrounds}

\usepackage{hyperref}

\usepackage[accepted]{icml2026}

\usepackage{amsmath}
\usepackage{amssymb}
\usepackage{amsthm}

\usepackage[capitalize,noabbrev]{cleveref}

\theoremstyle{plain}
\newtheorem{theorem}{Theorem}[section]
\newtheorem{proposition}[theorem]{Proposition}
\newtheorem{lemma}[theorem]{Lemma}

\theoremstyle{definition}
\newtheorem{definition}[theorem]{Definition}

\theoremstyle{remark}
\newtheorem{remark}[theorem]{Remark}

\usepackage[textsize=tiny]{todonotes}

\icmltitlerunning{AI-Assisted Discovery of Convex Relaxations via Dual Agents}
\begin{document}

\twocolumn[
  \icmltitle{{AI-Assisted Discovery of Convex Relaxations via Dual Agents}}



  \icmlsetsymbol{equal}{*}

  \begin{icmlauthorlist}
    \icmlauthor{Sungyoon Kim}{sch}
    \icmlauthor{Mert Pilanci}{sch}
  \end{icmlauthorlist}

  \icmlaffiliation{sch}{Department of Electrical Engineering, Stanford University, CA, USA}

  \icmlcorrespondingauthor{Sungyoon Kim}{sykim777@stanford.edu}

  \icmlkeywords{Machine Learning, ICML}

  \vskip 0.3in
]



\printAffiliationsAndNotice{}  

\begin{abstract}
Recent work shows that LLM agents can improve sharp-constant inequalities by
searching for extremal constructions, which yield \emph{upper} bounds. We address
the complementary side: a \emph{lower} bound holds for every admissible function and
follows from a convex relaxation of the nonconvex problem, with tighter relaxations
giving stronger bounds. We instantiate the autoresearch paradigm to discover such
relaxations: a coding agent proposes valid tightening constraints, a theory agent
verifies each one and searches for counterexamples, and every reported bound is
certified by an explicit dual-feasible point checked in rigorous interval
arithmetic. On two optimization constants studied by \citet{tao2025alphaevolve} - the
first autocorrelation inequality ($C_{6.2}$) and the Erd\H{o}s minimum-overlap
constant ($C_{6.5}$) - we improve the certified lower bounds from $1.28$ to $1.2937$
and from $0.379005$ to $0.37912$, respectively.
\end{abstract}
 
\section{Introduction}

Large language model (LLM) agents have begun to contribute to mathematical research in ways that were implausible only a short time ago. On competition mathematics, verification-and-refinement pipelines now reach gold-medal performance at the International Mathematical Olympiad~\citep{huang2025winning}, and frontier models have moved beyond contests toward genuine research: GPT-5 has produced new, human-verified results across several areas of mathematics, including settling previously open Erdős problems~\citep{openai2025science}. Larger-scale efforts point the same way - \citet{tsoukalas2026advancing} report solving 9 of 353 open Erdős problems with a frontier agent under Lean - based formal verification, while contamination - free benchmarks built from unpublished problems, most prominently \emph{First Proof}~\citep{firstproof2026}, show models solving research-level mathematics rather than retrieving memorized patterns. Taken together, these episodes indicate that agentic AI is no longer only a tool for routine assistance, but a participant in mathematical discovery.

Part of what makes these systems interesting is that their strengths differ from those of human mathematicians. An agent does not tire, can pursue a search through the night, and can ingest a body of literature far larger than any individual could hold in mind at once. These differences suggest that AI can contribute to mathematics in a different way from humans - not by imitating human workflows, but by playing to strengths humans lack. A clear example arises in functional analysis, where evolutionary and LLM-guided search has been used to improve bounds on autocorrelation inequalities and the related Erd\H{o}s minimum
overlap problem~\citep{tao2025alphaevolve, tttdiscover2026, einsteinarena2026, novikov2025alphaevolve}. There, the goal is to \emph{construct}
an extremal function: an explicit example that witnesses a bound. Because any
single admissible function yields a valid bound, the task reduces to optimizing a
score over a vast, highly nonconvex space of candidate functions-precisely the
regime in which an agent can act as a tireless and creative \emph{discoverer},
finding constructions that human search would reach only slowly, if at all.

We pursue a different way of leveraging these tools. Our work is motivated by two observations: AI-assisted coding has matured to the point where agents write competent, nontrivial programs \citep{jimenez2024swe}, and agents are increasingly able to solve hard algorithmic problems \citep{jain2025livecodebench, zheng2026livecodebench}. We use these capabilities to improve the \emph{other} side
of these inequalities-the \emph{lower} bounds. Unlike an upper bound, which
follows from exhibiting one function, a lower bound must hold against
\emph{every} admissible $f$; one obtains such a bound by solving a convex
relaxation of the original nonconvex problem, and a tighter relaxation yields a
stronger bound (we make this precise in \cref{sec:prelim}). Finding a good relaxation is itself an iterative, open-ended task: an agent repeatedly edits a program, benchmarks it, and tries to beat its own best result. This is a natural fit for the \emph{autoresearch} paradigm~\citep{karpathy2026autoresearch}, as seen in self-improving coding agents~\citep{robeyns2025selfimproving} and LLM-driven program search~\citep{romera2024mathematical}. We instantiate this paradigm for the discovery of convex relaxations
that certify lower bounds. A difficulty particular to our setting is that an agent may simply be \emph{wrong} - it may propose a relaxation whose constraints are not actually valid for all $f$, or commit prematurely to a single line of attack. The correctness of the convex problem is the central risk: an invalid constraint yields a bound that is not a bound at all. We therefore pair the coding
agent with a \emph{theory agent} whose role is to read the proposed relaxation, rigorously establish that each constraint is valid, and verify the accompanying argument before any improvement is accepted.

Using this dual-agent loop, we improve the certified lower bounds on two autocorrelation inequalities studied by \citet{tao2025alphaevolve} - $C_{6.2}$, arising
in additive combinatorics and $C_{6.5}$, controlling the asymptotics of the Erd\H{o}s
minimum overlap problem (see \cref{sec:prelim} for definitions). For the first, we raise the lower bound from $1.28$ to $1.2937$; for the second, we improve the
best known lower bound of $0.379005$ to $0.37912$. To
our knowledge, this is the first use of an autoresearch-style agentic loop to improve \emph{certified lower bounds} on these inequalities: rather than searching for a better extremal example, the agents search for a better \emph{program} - a convex relaxation-whose validity is established by mathematical proof and whose bound is certified by duality.

\section{Preliminaries}
\label{sec:prelim}

\subsection{Autocorrelation Inequalities}
\label{sec:prelim-autocorr}

For absolutely integrable $f, g : \mathbb{R} \to \mathbb{R}$, the convolution is
$(f * g)(t) = \int_{\mathbb{R}} f(x) g(t-x)\, dx$; when $g$ equals $f$ or a
reflection of $f$, we refer to such expressions as \emph{autocorrelations} \citet{tao2025alphaevolve}.
Sharp constants in functional inequalities involving autocorrelations have been
studied extensively; we refer to \citet{tao2025alphaevolve} for a survey and for the problem numbering we adopt below, and \citet{optimization-constants-repo} for the history of bounds for these problems. We consider two such constants.

\paragraph{The first autocorrelation inequality (Problem 6.2 of \cite{tao2025alphaevolve}).}
Let $C_{6.2}$ denote the largest constant for which
\begin{equation}
\max_{-1/2 \le t \le 1/2} \int_{\mathbb{R}} f(t-x) f(x)\, dx
\;\ge\; C_{6.2} \left( \int_{-1/4}^{1/4} f(x)\, dx \right)^2
\label{eq:prob62}
\end{equation}
holds for all non-negative $f : \mathbb{R} \to \mathbb{R}$. This constant arises
in additive combinatorics, in connection with the size of Sidon sets \cite{cloninger2017suprema, martin2009supremum}. The
previously known bounds are
\begin{equation}
1.28 \;\le\; C_{6.2} \;\le\; 1.502862,
\end{equation}
where the lower bound is obtained by analytic/duality arguments~\citep{cloninger2017suprema}  and the upper bound by an
explicit counterexample~\citep{einsteinarena2026}. 


\paragraph{The Erd\H{o}s minimum overlap problem (Problem 6.5).}
Let $C_{6.5}$ denote the largest constant for which
\begin{equation}
\sup_{x \in [-2,2]} \int_{-1}^{1} f(t)\, g(x+t)\, dt \;\ge\; C_{6.5}
\label{eq:prob65}
\end{equation}
for all non-negative $f, g : [-1,1] \to [0,1]$ with $f + g = 1$ on $[-1,1]$ and
$\int_{-1}^{1} f = 1$. Outside this interval $[-1,1]$, $f(x) = g(x) = 0$. This
constant governs the asymptotics of the minimum overlap problem of
Erd\H{o}s~\citep{tao2025alphaevolve}. The known bounds are
\begin{equation}
0.379005 \;\le\; C_{6.5} \;\le\; 0.38087131058,
\end{equation}
where the lower bound was obtained via a family of convex optimization problems \cite{white2022erd} and the upper bound by
a step-function construction~\citep{einsteinarena2026}.

In both problems the quantity of interest is an infimum over an
infinite-dimensional set of admissible functions, of a score that is nonconvex in $f$. A lower bound on the constant is therefore a statement that \emph{no}
admissible function drives the score below a given value-a universally
quantified claim, in contrast to the existential claim established by a single
construction.

\subsection{Convex Relaxation and Certified Bounds}
\label{sec:prelim-relax}

We briefly fix the vocabulary of relaxation and duality used throughout. Consider
a nonconvex problem
\begin{equation}
p^\star \;=\; \inf_{f \in \mathcal{F}} \; J(f),
\label{eq:nonconvex}
\end{equation}
where $\mathcal{F}$ is the (infinite-dimensional) set of admissible functions and
$J$ is the score. A \emph{relaxation} replaces \eqref{eq:nonconvex} by a convex
problem
\begin{equation}
r^\star \;=\; \inf_{z \in \mathcal{C}} \; \tilde{J}(z),
\label{eq:relax}
\end{equation}
constructed so that every admissible $f$ induces a feasible $z \in \mathcal{C}$
with $\tilde{J}(z) \le J(f)$. Any such relaxation satisfies $r^\star \le
p^\star$, so $r^\star$ is a valid lower bound on the constant of interest. The
relaxation is \emph{valid} precisely when this inclusion holds-that is, when
every constraint defining $\mathcal{C}$ is satisfied by the encoding of every
admissible $f$. A \emph{tighter} relaxation is one with a larger $r^\star$
(equivalently, a smaller feasible set that still contains all admissible
encodings); tightening the relaxation improves the bound, up to the limit
$r^\star = p^\star$.

\paragraph{Certified bounds via duality.}
Solving \eqref{eq:relax} numerically returns a value, but for a \emph{certified}
bound we appeal to weak duality. For any dual-feasible point $\lambda$, the dual
objective $d(\lambda)$ satisfies
\begin{equation}
d(\lambda) \;\le\; r^\star \;\le\; p^\star
\label{eq:weakduality}
\end{equation}
Exhibiting a single dual-feasible $\lambda$ with $d(\lambda) = \beta$ therefore proves
$p^\star \ge \beta$, and this certificate is checkable independently of the solver
that produced it: one need only verify that $\lambda$ satisfies the dual
constraints and evaluate $d(\lambda)$. This is the sense in which our bounds are
certified, and it is what distinguishes a proven lower bound from a numerically
reported optimal value. We return to how we find a rigorous dual certificate and a lower bound in \cref{sec:method}.

\section{Method}
\label{sec:method}
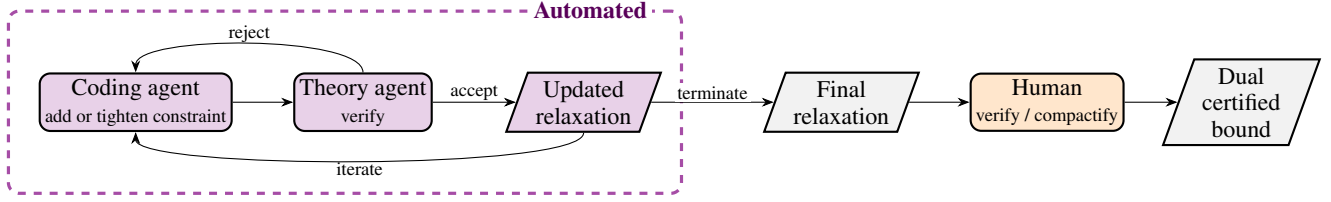
\begin{figure*}[t]
\centering
\begin{tikzpicture}[
  font=\footnotesize,
  >={Stealth[length=1.6mm]},
  agent/.style={draw, rounded corners, minimum height=7mm, minimum width=15mm,
                align=center, thick, inner sep=2pt, fill=violet!18},
  relax/.style={draw, trapezium, trapezium left angle=70, trapezium right angle=110,
                align=center, thick, inner sep=3pt, fill=violet!18, minimum height=7mm},
  human/.style={draw, rounded corners, minimum height=7mm, minimum width=15mm,
                align=center, thick, inner sep=2pt, fill=orange!20},
  artifact/.style={draw, trapezium, trapezium left angle=70, trapezium right angle=110,
                   align=center, thick, inner sep=3pt, fill=gray!10, minimum height=7mm},
  every node/.style={align=center},
]

\node[agent] (coding) {Coding agent\\{\scriptsize add or tighten constraint}};
\node[agent, right=8mm of coding] (theory) {Theory agent\\{\scriptsize verify}};
\node[relax, right=11mm of theory] (program) {Updated\\relaxation};

\draw[->] (coding) -- (theory);
\draw[->] (theory) -- node[above,inner sep=0.5pt]{\scriptsize accept} (program);
\draw[->] (theory.north) .. controls +(0,5mm) and +(0,5mm) ..
      node[above,inner sep=0.5pt]{\scriptsize reject} (coding.north);
\draw[->] (program.south) .. controls +(0,-5mm) and +(0,-5mm) ..
      node[below,inner sep=0.5pt]{\scriptsize iterate} (coding.south);

\begin{scope}[on background layer]
  \node[draw=violet!70, dashed, very thick, rounded corners,
        fit=(coding)(theory)(program), inner xsep=4mm, inner ysep=8mm] (loopbox) {};
\end{scope}
\node[violet!70!black, font=\small\bfseries, anchor=east, fill=white, inner sep=2pt]
      at ([xshift=-4mm]loopbox.north east) {Automated};

\node[artifact, right=12mm of loopbox.east, anchor=west] (final) {Final\\relaxation};
\node[human, right=8mm of final] (human) {Human\\{\scriptsize verify / compactify}};
\node[artifact, right=7mm of human] (out) {Dual\\certified\\bound};

\draw[->] (program.east) -- node[above,inner sep=1pt,fill=white]{\scriptsize terminate} (final.west);
\draw[->] (final) -- (human);
\draw[->] (human) -- (out);

\end{tikzpicture}
\caption{The framework. The \emph{automated} block (purple, dashed) alternates
between a coding agent that adds or tightens a constraint and a theory agent that
verifies its validity and searches for counterexamples; accepted constraints
update the relaxation, and the loop repeats. On termination it yields the final
relaxation, which a human verifies and compactifies; the resulting bound is
certified by an explicit dual-feasible point.}
\label{fig:loop}
\end{figure*}

\subsection{Problem Formulation}
\label{sec:method-formulation}

Each constant in \cref{sec:prelim-autocorr} is the value of a min-max problem over an infinite-dimensional class of admissible functions. We describe the
convex problem formulation through the minimum overlap problem, following
\citet{white2022erd}; the autocorrelation inequality $C_{6.2}$ of
\eqref{eq:prob62} is treated analogously.

For the minimum overlap problem, the constant is the infimum over admissible $f$ of a maximum over translates,
\begin{equation}
C_{6.5} \;=\; \inf_{f}\;\max_{t \in [-2, 2]}\;\int_{-1}^{1} f(x)g(x+t)\,dx,
\label{eq:overlap-minmax}
\end{equation}
taken over measurable $f:[-1,1] \rightarrow [0,1]$ and normalization $\int_{-1}^{1} f = 1$. White's reduction turns this nonconvex min--max into a convex program of a finite
list of variables. Write $M(t)=\int_{-1}^{1} f(x)g(x+t)\,dx$ for the overlap at $t$, where $g(x) = 1-f(x)$ for $x \in [-1,1]$ and 0 otherwise, and partition $[0, 2]$ into grid cells $I_1,\dots,I_N$ of width
$L$. White's reduction then introduces the following variables: a scalar $\Omega$ bounding $\max_k M(k)$; the cell integrals
of the overlap on the two shift directions,
\begin{equation}
w_j \;=\; \frac{1}{L} \int_{I_j} M(k)\,dk, \quad v_j \;=\; \frac{1}{L} \int_{I_j} M(-k)\,dk,
\label{eq:wv-def}
\end{equation}
the cosine and sine Fourier coefficients $c_k,d_k$ of $f$ and $a_m,b_m$ of $M$; and
truncation slacks $\varepsilon_m,\delta_m$ that absorb the Fourier modes beyond the
truncation order so that finitely many coefficients suffice.
 
Admissibility of $f$ imposes constraints on these variables, and the essential
point is that each one must be a genuine consequence of $0\le f\le 1$ and
$\int_{-1}^{1} f=1$, so that no admissible $f$ is excluded. Two are immediate. The property $\int_{-2}^{2} M(x)dx=1$ gives the linear identity $L\sum_j(w_j+v_j)=1$, and because
$\Omega$ bounds the overlap at every translate, each cell integral obeys $w_j,v_j\le \Omega$. The remaining conditions - mean and second-moment bounds, the Fourier identities linking $a_m,b_m$ to $c_k,d_k$, and the bounds on the truncation slack $\epsilon_m, \delta_m$, are derived in the same spirit; we refer the reader to
\citet{white2022erd} or \cref{alg:overlap} for the full list. We retain only those conditions that are convex in the variables to make the problem convex. \textbf{Let $\mathcal{K}$ be the feasible set they cut out.}
 
The lower bound now follows from a single observation. Fix any admissible $f$ and
read off its variables: take $\Omega=\max_{k\in [-2,2]} M(k)$, let $w_j,v_j$ be the cell
integrals in \eqref{eq:wv-def}, and let the remaining entries be the Fourier data of
$f$. By construction this assignment satisfies every constraint of $\mathcal{K}$, so
it is a feasible point of the convex program, and its objective value is exactly
$\max_{k\in[-2,2]} M(k)$, the overlap score of $f$. The convex program minimizes that same
objective over all of $\mathcal{K}$, returning a value $\Omega^\star$. Since every
admissible $f$ supplies a feasible point whose objective equals its own score, the
minimum can only be smaller: $\Omega^\star\le \max_{k\in[-2,2]} M(k)$ for every admissible $f$,
and therefore $\Omega^\star\le C_{6.5}$. 

\subsection{Successive Convex Relaxation}
\label{sec:method-successive}

The quality of the convex problem is governed entirely by the constraint set $\mathcal{K}$. Tightening $\mathcal{K}$-adding a valid condition, or
strengthening an existing one-shrinks the feasible set toward the true set of
admissible coefficient vectors, and the minimum can only increase:
\begin{equation}
\mathcal{K}' \subseteq \mathcal{K}
\quad\Longrightarrow\quad
\min_{c \in \mathcal{K}'} \Phi(c) \;\ge\; \min_{c \in \mathcal{K}} \Phi(c).
\label{eq:monotone}
\end{equation}
So long as every condition is valid, $\mathcal{K}'$ still contains all admissible
encodings, and the larger value remains a lower bound on $\Omega$. Iterating
yields a monotone sequence of valid lower bounds increasing toward $\Omega$.

We stress that ``tightening'' here means \emph{adding valid constraints or improving existing ones}, not
naively increasing resolution. One could enlarge the model with more Fourier modes or
finer step functions; this refines the discretization but does not, in itself,
tighten the relaxation in the sense of \eqref{eq:monotone}. The improvements we
seek are mathematical: properties of admissible coefficient sequences that a
coarser formulation failed to impose. The task at each iteration is therefore to
\emph{identify a valid constraint that previous relaxations overlooked}, or to
sharpen one imposed too weakly. The concrete constraints discovered for each
problem, and their effect on the bound are reported in \cref{sec:results}.

\subsection{The Dual-Agent Loop}
\label{sec:method-loop}

We search for such constraints with a human-in-the-loop framework: a loop of two
agents drives the discovery, and a human conducts a final review and
compactification. The loop instantiates the autoresearch paradigm
\citep{karpathy2026autoresearch} for the discovery of certified relaxations. We used Claude Opus 4.6 for our agent experiments, and detailed prompts can be found in \cref{AppxD}.

\paragraph{Coding agent.} The coding agent proposes an improvement to the current
relaxation-typically a new constraint on the coefficient vector $c$, or a
strengthening of an existing one-together with a rigorous proof that the constraint is valid, i.e.\ that it holds for the coefficients of every admissible $f$. It
then implements the modified convex program and solves it, producing a candidate
bound.

\paragraph{Theory agent.} The theory agent reads the proposed constraint and its proof and checks that the constraint is a sound consequence of the admissibility of $f$. Critically, it does not only attempt to confirm the argument: it also \emph{actively searches for a counterexample}-an admissible $f$ whose coefficients violate the proposed constraint. A constraint survives only
if its justification holds and no counterexample is found. This guards the invariant on which the entire bound rests: a constraint that some admissible $f$ violates would shrink $\mathcal{K}$ below the admissible set and silently invalidate the lower bound.

The prompts used for the coding agent and the theory agent, as well as the problem descriptions, are in \cref{AppxD}.

\paragraph{Human review and compactification.} The agent loop tends to accumulate
constraints, not all of which are necessary: many are redundant at the optimum and
contribute nothing to the bound. Once the loop has produced a candidate
relaxation, a human performs a final pass-reading the program and the validity
arguments, and removing redundant constraints to obtain a minimal program.
Removing a constraint enlarges $\mathcal{K}$ and so, by \eqref{eq:monotone}, can
only lower the bound; a constraint is therefore dropped only when the certified
value was not meaningful in the history of objective values.

\paragraph{Certifying Correctness} Two conditions must hold for a reported bound to be valid, and they are established by different means. First, every constraint in the final program must be a true property of all
admissible $f$. This is established by a mathematical argument, produced by the agents and confirmed by human review; we do \emph{not} formally verify these
arguments (e.g.\ in a proof assistant), so this is the step at which correctness
ultimately rests on human judgment.
Second, the reported number must actually be a lower bound for the resulting convex program. This does \emph{not} rest on human judgment: given the program, its dual is formed by the standard duality recipe, and any dual-feasible point yields a valid bound by weak duality (\cref{sec:prelim-relax}). We certify each bound by exhibiting such a point, rigorously checking whether the point is dual-feasible, and evaluating the dual objective in interval arithmetic. 

We
work in the canonical conic form used by CVXPY \cite{diamond2016cvxpy}. Every program in this paper can be
written as
\begin{equation}
\min_{x}\; c^\top x \quad\text{subject to}\quad Ax + s = b,\;\; s\in\mathcal{K},
\label{eq:conic-primal}
\end{equation}
where $\mathcal{K}$ is a product of the nonnegative orthant, second-order cones, and
positive-semidefinite cones; CVXPY's canonicalization produces exactly this form.
The corresponding dual is
\begin{equation}
\max_{y}\; -b^\top y \quad\text{subject to}\quad A^\top y + c = 0,\;\; y\in\mathcal{K}^{*},
\label{eq:conic-dual}
\end{equation}
with $\mathcal{K}^{*}$ the dual cone, and weak duality gives, for any dual-feasible
$y$,
\begin{equation}
-b^\top y \;\le\; \min_x c^\top x \;=\; \Omega^{\star}.
\label{eq:weak-dual-conic}
\end{equation}
A single dual-feasible $\tilde y$ therefore proves $\Omega^\star \ge -b^\top\tilde
y$, independently of the solver that produced it. We obtain a candidate witness
$\tilde y$ from CVXPY and verify it rigorously, which raises two issues: $\tilde y$ is only numerically dual-feasible, and controlling the rounding introduced by irrational problem data.
 
Cone membership is the first requirement. For the orthant and second-order blocks it
reduces to a finite set of sign and norm inequalities, checked directly. The
positive-semidefinite blocks are more delicate: certifying that a matrix is PSD in
the presence of rounding cannot be done by computing eigenvalues naively, since a
matrix that is PSD in exact arithmetic may have a small negative computed
eigenvalue. We instead certify PSD-ness by a verified factorization carried out in
interval arithmetic; the details are deferred to Appendix~\ref{app:overlap}.
 
The second issue is rounding. The coefficients $A,b,c$ contain irrational entries
(trigonometric moments and Fej\'er weights), so the matrices actually stored,
$\bar A,\bar b,\bar c$, are rounded; writing $\varepsilon_A,\varepsilon_b,
\varepsilon_c$ for elementwise bounds on $\|A-\bar A\|_\infty$, $\|b-\bar b\|_\infty$, and
$\|c-\bar c\|_\infty$, each true coefficient is enclosed in an interval of the
corresponding width. Both the stored problem and the final objective therefore carry
rounding, and a rigorous bound must absorb both. Let $r$ bound the dual residual
$\|A^\top\tilde y + c\|_\infty$, the extent to which $\tilde y$ fails exact dual
feasibility, and let $X$ be an a priori bound on $\|x\|$ over primal-optimal $x^{*}$.
Weak duality together with these error terms yields the certified bound
\begin{equation}
\Omega^{\star} \;=\; c^\top x^{\star}
\;\ge\; -b^\top \tilde y
\;-\;\bigl(r + \varepsilon_A\|\tilde y\|_1 + \varepsilon_c\bigr)\,X
\;-\;\varepsilon_b\|\tilde y\|_1.
\label{eq:certified-bound}
\end{equation}
Each quantity on the right is
a guaranteed bound, and using the bound leads to a rigorous lower bound of $\Omega^{*}$. We evaluate the entire
right-hand side in interval arithmetic with outward rounding, so the reported value is the lower endpoint of a verified enclosure. See \cref{AppxC} for a detailed derivation.

\section{Results}
\label{sec:results}

We apply the dual-agent approach described in \cref{sec:method} to two constants $C_{6.2}$ and $C_{6.5}$
For each we report the convex program discovered by the loop, and the certified lower bound it yields. All bounds are certified by an explicit dual-feasible point
evaluated in exact rational arithmetic (Section~\ref{sec:method}); no
floating-point margin enters the final certificate. At last, an analysis of the loop's behavior for the first autocorrelation inequality is discussed in \cref{sec:results-loop}.
 
\subsection{Discovered convex programs}
\label{sec:results-programs}
 
For $C_{6.5}$ the loop took the convex program of \citet{white2022erd} as its starting point and searched for valid constraints to tighten it. The program
(Algorithm~\ref{alg:overlap}) shows the resulting convex program, with the \colorbox{boxpink}{boxed} pair of constraints are the ones added by the loop.
The discovered constraints are a pair of \emph{Bochner conditions}:
$T_f\succeq0$ enforces that the moment sequence of $f$ is realizable by a nonnegative measure, and $I-T_f\succeq0$ enforces the same for its complement
$g=1-f$. White's formulation does not explicitly use the constraint $0 \leq f \leq 1$ pointwise, where our advantage is introduced. The loop identified both conditions automatically, and the theory
agent verified their validity (Bochner's theorem, applied to $f$ and to $g=1-f$) before they were accepted. For a rigorous proof of the added constraint see \cref{AppxA}.

The convex program depends on six parameters
$h_1,h_2,p_1,p_2,q_1,q_2$ that bracket three quantities of an admissible $f$: the
mean $E(M)$ of the overlap, with $h_1\le E(M)\le h_2$; the first cosine coefficient
$c_1$ of $f$, with $p_1\le c_1\le p_2$; and the first sine coefficient $d_1$, with
$q_1\le d_1\le q_2$. Over all admissible $f$ these quantities range within
$E(M)\in[0,2]$, $c_1\in[0,1]$, and $d_1\in[-1,1]$ (for a proof see \cite{white2022erd}). Write
$\Omega(h_1,h_2,p_1,p_2,q_1,q_2)$ for the certified optimum of the program with the
parameters constrained to the box $[h_1,h_2]\times[p_1,p_2]\times[q_1,q_2]$. For any
such box the program is a valid relaxation, so its optimum lower-bounds the score of
every admissible $f$ whose $(E(M),c_1,d_1)$ falls in that box.
 
Since every admissible $f$ has $(E(M),c_1,d_1)$ somewhere in the full box
$\mathcal{B}=[0,2]\times[0,1]\times[-1,1]$, solving the program over all of
$\mathcal{B}$ at once already gives a valid lower bound on $\Omega$. That bound is
weak, however: a relaxation over a large parameter box is loose, so its optimum
comes out well below $\Omega$. The bound sharpens as the box shrinks - constraining the
parameters to a small region tightens the relaxation and raises its certified
optimum. We therefore split $\mathcal{B}$ into sub-boxes and bound each one
separately, which is the branch-and-bound step. Cover $\mathcal{B}$ by sub-boxes
$B'=[h_1,h_2]\times[p_1,p_2]\times[q_1,q_2]$; a single certified dual point for
$\Omega(h_1,h_2,p_1,p_2,q_1,q_2)$ lower-bounds the objective \emph{simultaneously for
every} $(E(M),c_1,d_1)\in B'$, and the smaller $B'$ is, the tighter this local bound.
A sub-box whose certificate already meets the target is cleared; otherwise it is
split further and the procedure recurses, until all of $\mathcal{B}$ is covered. The
reported bound is the smallest certified optimum over the cover - a valid lower bound
on $\Omega$, and a sharper one than the single whole-box solve.
 
 
For $C_{6.2}$ no prior convex program was available; the loop constructed one from scratch (Algorithm~\ref{alg:autocorr}). The agentic system introduced variables as the following: a discretized nonnegative measure $p$ on $[-\tfrac14,\tfrac14]$ (the support
reduction of Lemma~\ref{lem:support}), with trigonometric moments $a_k,b_k$. The
program combines: a probability simplex and a positive-semidefinite moment matrix;
per-cell Fourier envelopes bracketing $a_k,b_k$, computed in exact integer
arithmetic so envelope validity carries no floating-point tolerance; a
Fej\'er--Riesz Gram matrix $Q$ certifying $\Omega-(F_K*g)(t)\ge0$ for all $t$ (the
sum-of-squares encoding of the pointwise lower bound on the autocorrelation peak,
$g=f*f$); a quadratic energy bound following from
$\int g^2=\sum_m|\hat f(m)|^4$; and plain and \emph{localized} Bochner conditions,
the latter valid because $\cos(2\pi x)\ge0$ exactly on $[-\tfrac14,\tfrac14]$. The derivation of the novel convex program is in \cref{AppxB}.
 
\begin{algorithm}[!ht]
\caption{Discovered convex program for the Erd\H{o}s minimum-overlap constant $C_{6.5}$.}
\label{alg:overlap}
\begin{algorithmic}[1]
\REQUIRE grid size $N$ ($L=2/N$), truncation $T$, modes $R$; box $h_1,h_2,p_1,p_2$ and fixed $q_1,q_2$
\STATE \textbf{minimize} $\Omega$
\STATE over $\Omega\le1$, $\{w_j,v_j\}$, $\{c_k,d_k\}_{k=1}^T$, $\{a_m,b_m\}_{m=1}^{2R}$, $\{\varepsilon_m,\delta_m\}_{m=1}^R$
\STATE \textbf{subject to}
\STATE $0\le w_j,\,v_j\le \Omega\le 1$
\STATE $L\sum_j (w_j+v_j)=1$
\STATE $h_1\le L^2\sum_j(jw_j-(j-1)v_j)$
\STATE $L^2\sum_j((j-1)w_j-jv_j)\le h_2$
\STATE $\frac23+\frac{h_1^2}{2}\le L^3\sum_j j^2(w_j+v_j)$
\STATE $L^3\sum_j(j-1)^2(w_j+v_j)\le \frac23+\frac{h_2^2}{2}$
\STATE $\frac{L}{2}\sum_j\alpha^-_{j,m}(w_j+v_j)\le
       \frac{4\sin(\pi m/2)}{\pi m}a_m-2(a_m^2+b_m^2)$
\STATE $\frac{L}{2}\sum_j(\beta^-_{j,m}w_j-\beta^+_{j,m}v_j)\le
       -\frac{4\sin(\pi m/2)}{\pi m}b_m$
\STATE $-\frac{4\sin(\pi m/2)}{\pi m}b_m\le
       \frac{L}{2}\sum_j(\beta^+_{j,m}w_j-\beta^-_{j,m}v_j)$
\STATE $a_m=\frac12c_{m/2}$, $b_m=\frac12d_{m/2}$ \hfill ($m$ even)
\STATE $a_m=\varepsilon_m+\frac{2m\sin(\pi m/2)}{\pi}
       \left(\frac{1}{2m^2}+\sum_{k=1}^{T}\frac{(-1)^k}{m^2-4k^2}c_k\right)$
\STATE $b_m=\delta_m+\frac{4}{\pi}
       \sum_{k=1}^{T}\frac{k(-1)^k\sin(\pi m/2)}{m^2-4k^2}d_k$ \hfill ($m$ odd)
\STATE $|\varepsilon_m|,|\delta_m|\le \text{tail bound}$ \hfill ($2m-1<2T$)
\STATE $|c_k|,|d_k|\le \frac{2}{\pi}$, $\sum_k(c_k^2+d_k^2)\le \frac12$
\STATE $p_1\le c_1\le p_2$, $q_1\le d_1\le q_2$
\STATE $\frac{L}{2}\sum_j\alpha^+_{j,2}(w_j+v_j)\ge
       -\frac12(p_2^2+\max\{q_1^2,q_2^2\})$
\STATE \colorbox{boxpink}{$T_f\succeq0,\qquad I-T_f\succeq0$}
\STATE \colorbox{boxpink}{$(T_f)_{kl}=\frac14\delta_{kl}
       +\frac12(a_{|k-l|}-\mathrm{\textbf{i}}\operatorname{sgn}(k-l)b_{|k-l|})$}
\STATE \textbf{Branch and bound:} Fix $(q_1,q_2)$ and sweep the $(h,p)$ rectangle. Certify each rectangle by an exact-arithmetic dual point; if certified, clear it, otherwise split and recurse.
return smallest certified value over cleared rectangles
\end{algorithmic}
\end{algorithm}

\begin{algorithm}[!ht]
\caption{Discovered semidefinite program for the Sidon-set autocorrelation constant $C_{6.2}$.}
\label{alg:autocorr}
\begin{algorithmic}[1]
\REQUIRE cells $2N$ on $[-\frac14,\frac14]$, modes $K$
\STATE \textbf{minimize} $\Omega$
\STATE over $\Omega\ge0$, $p\in\mathbb{R}^{2N}_{\ge0}$, moments $a_k,b_k$, $M\succeq0$, Hermitian $Q\succeq0$, slacks $v_k\ge0$
\STATE \textbf{subject to}
\STATE $\sum_j p_j=1$
\STATE $M\succeq0$, $M_{00}=1$, $M_{0k}=a_k$, $M_{0,K+k}=b_k$
\STATE $\gamma^-_k\cdot p\le a_k\le\gamma^+_k\cdot p$, $\sigma^-_k\cdot p\le b_k\le\sigma^+_k\cdot p$
\STATE $\Omega\ge 1+2\sum_k w_kv_k$, $v_k\ge(M_{kk}+M_{K+k,K+k})^2$
\STATE $Q\succeq0$, $\operatorname{tr}(Q)=\Omega-1$
\STATE $\sum_i Q_{i,i+k}=\mathrm{\textbf{i}}\,2w_kM_{k,K+k}-w_k(M_{kk}-M_{K+k,K+k})$
\STATE $\begin{bmatrix}T_R&-T_I\\T_I&T_R\end{bmatrix}\succeq0$,
       $T_{ij}=a_{|i-j|}-\mathrm{\textbf{i}}\operatorname{sgn}(i-j)b_{|i-j|}$
\STATE $\begin{bmatrix}T^\nu_R&-T^\nu_I\\T^\nu_I&T^\nu_R\end{bmatrix}\succeq0$, \\
$(T^\nu)_{ij}=\nu_{i-j}$, $\nu_k=\frac12(\hat\mu_{k-1}+\hat\mu_{k+1})$
\end{algorithmic}
\end{algorithm}
 
\subsection{Improved bounds}
\label{sec:results-bounds}
 
Both constants have long histories of incremental improvement, and in both cases
the recent gains on the \emph{upper} (construction) side came from search-based
methods, while the \emph{lower} (duality) side had not moved since 2022. Our
framework improves the lower side of both.

\paragraph{Minimum overlap ($C_{6.5}$).}
The minimum-overlap constant $C_{6.5}=\lim_n M(n)/n$ has an even longer record. Early
lower bounds came from averaging and rearrangement arguments - Erd\H{o}s'
$\tfrac14$ \cite{erdHos1955some}, and Moser's
$\sqrt{4-\sqrt{15}}\approx0.35639$ \cite{moser1966overlap} - until \citet{white2022erd}
reformulated the lower bound as a convex program and obtained $0.379005$. The upper
side has been refined by step-function constructions: $0.382002$ \cite{haugland1996advances}
and $0.380926$ \cite{haugland2016minimum}, then $0.380924$ AlphaEvolve \citep{tao2025alphaevolve},
$0.380876$ \cite{tttdiscover2026} and $0.380871$ in \cite{einsteinarena2026}. Our augmented program certifies
\[
  0.37912 \leq C_{6.5}\ ,
\]
improving \citeauthor{white2022erd}'s $0.379005$. The bound is the minimum certified value over the branch-and-bound sweep
of the $(h,p)$ box; the subdivision tree is shown in Figure~\ref{fig:bb-tree} (Appendix~\ref{app:overlap}). Table~\ref{tab:c1b} summarizes the record.
 
\begin{table}[!ht]
\caption{State of the art for the Erd\H{o}s minimum-overlap constant $C_{6.5}$
(Eq.~\eqref{eq:prob65}). Lower bounds are the duality side; upper bounds come from
step-function constructions and recent search-based methods. Records as in
\citet{optimization-constants-repo}.}
\centering
\small
\setlength{\tabcolsep}{4pt}
\renewcommand{\arraystretch}{1.25}
\begin{tabular}{@{}p{0.45\linewidth}p{0.45\linewidth}@{}}
\toprule
Lower bounds & Upper bounds \\
\midrule
$0.25$ \cite{erdHos1955some} & $0.380926$~\citep{haugland2016minimum}\\
$0.35639$~\citep{moser1966overlap}               &  $0.380924$~\citep{tao2025alphaevolve} \\
$0.379005$~\citep{white2022erd}   
& $0.380876$~\cite{tttdiscover2026} \\
$\mathbf{0.37912}$~\textbf{(Ours)}    
&  $0.380871$ \cite{einsteinarena2026}\\
\bottomrule
\end{tabular}
\label{tab:c1b}
\end{table}  
 
\paragraph{Autocorrelation ($C_{6.2}$).}
$C_{6.2}$ has been bounded below through a sequence of
Fourier-analytic arguments: $1.262$~\citep{martin2009supremum} and
$1.2749$~\citep{matolcsi2010improved} - who also conjectured a ceiling of $1.276$
for their method - followed by $1.28$ via a finite-case relaxation
\citep{cloninger2017suprema}. The upper side, by explicit constructions, currently stands at $1.502862$~\citep{einsteinarena2026}. Our discovered SDP certifies
\[
  1.2937 \leq C_{6.2}\,
\]
improving the previous lower record of $1.28$. The certificate is an explicit dual-feasible point of the SDP, verified positive semidefinite and feasible in exact interval arithmetic, so the bound holds without reliance on the numerical solver or floating-point error analysis. Table~\ref{tab:c1a} summarizes the record.
 
\begin{table}[!ht]
\caption{State of the art for the Sidon-set autocorrelation constant $C_{6.2}$(Eq.~\eqref{eq:prob62}). Records as in
\citet{optimization-constants-repo}.}
\centering
\small
\setlength{\tabcolsep}{4pt}
\renewcommand{\arraystretch}{1.25}
\begin{tabular}{@{}p{0.45\linewidth}p{0.45\linewidth}@{}}
\toprule
Lower bounds & Upper bounds \\
\midrule
$1.262$~\citep{martin2009supremum}      & $1.503133$~\citep{wang2025thetaevolve} \\
$1.2749$~\citep{matolcsi2010improved}   & $1.503871$~\citep{ye2026evaluation} \\
$1.28$~\citep{cloninger2017suprema}     & $1.502870$~\citep{tttdiscover2026} \\
$\mathbf{1.2937}$~\textbf{(Ours)}       & $1.502862$~\citep{einsteinarena2026} \\
\bottomrule
\end{tabular}
\label{tab:c1a}
\end{table}
 
\subsection{Loop convergence and verification}
\label{sec:results-loop}

\begin{figure*}[t]
    \centering
    \includegraphics[width=\linewidth]{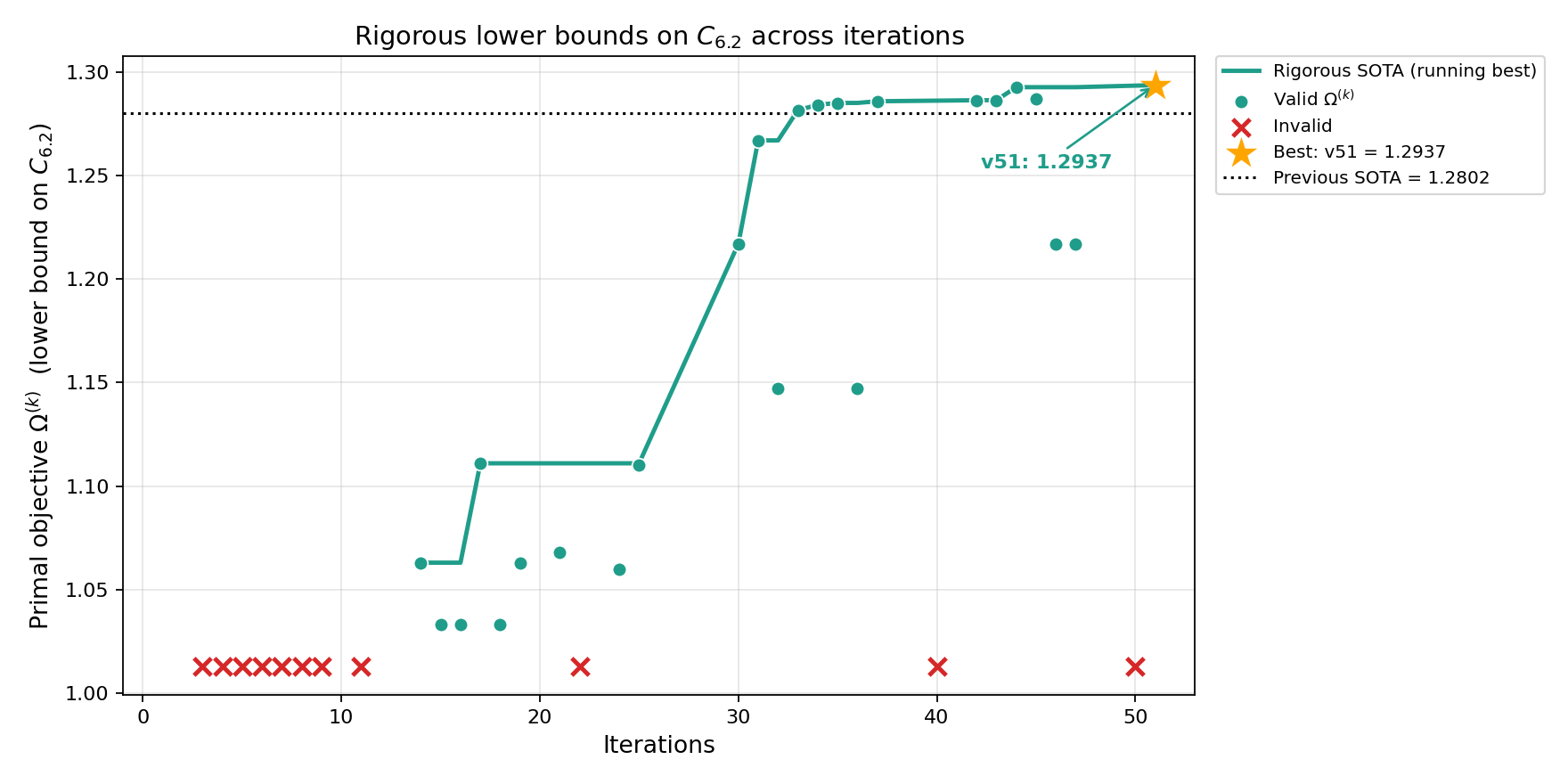}
    \caption{Certified primal objective of the convex program as a function of the loop iteration (version). The objective increases as the coding agent adds or tightens valid constraints; points in red are proposals the theory agent rejected as invalid (failing for some admissible $f$), as opposed to merely non-binding.}
\label{fig:loop-trajectory}
\end{figure*}

Figure~\ref{fig:loop-trajectory} plots the certified primal objective against the
loop iteration. The objective increases as the coding agent adds or tightens valid
constraints; points shown in red are proposals the verifier rejected as
\emph{invalid} (a constraint failing to hold for some admissible $f$, established by
an explicit counterexample), as opposed to merely non-binding. The loop makes progress on the certified bound while the verifier filters out invalid proposals that would otherwise corrupt it.
 
Beyond catching invalid \emph{constraints}, the theory agent also corrects the
search \emph{direction}. In one episode the coding agent restricted its candidate
functions to even functions, obtaining an invalid bound; the theory
agent flagged that the even restriction is a genuine loss of generality - it
discards admissible asymmetric configurations, so the resulting bound applies only
to the even-restricted constant and not to $C_{6.2}$ itself - and directed the next
iterations to drop it. After a few iterations the coding agent got rid of the even assumption: following the verifier's
recommendation, the coding agent got rid of the even function assumption and started looking for a genuine lower bound. For the raw interaction snippets see \cref{AppxF}.

\begin{figure}[t]
\centering
\begin{tikzpicture}[
  font=\footnotesize,
  codebubble/.style={fill=codeblue, text=codetext, rounded corners=6pt,
    align=left, inner sep=7pt, text width=56mm},
  theorybubble/.style={fill=theorylav, text=theorytext, rounded corners=6pt,
    align=left, inner sep=7pt, text width=56mm},
  speaker/.style={font=\scriptsize\bfseries, text=gray, inner sep=1pt},
]
\coordinate (L) at (0,0);
\coordinate (R) at (7.3,0);
 
\node[speaker, anchor=north west] (s1) at (L) {CODING AGENT};
\node[codebubble, anchor=north west] (b1) at ([yshift=-1pt]s1.south west)
  {Iteration~v7 assumes $f$ is even. The WLOG reductions and the
  windowed-integral constraints are proved; the bound rises monotonically to
  $1.344$ at $N{=}20$, strictly above $1$ for the first time.};
 
\node[speaker, anchor=north east] (s2) at ([yshift=-5mm]b1.south-|R) {THEORY AGENT};
\node[theorybubble, anchor=north east] (b2) at ([yshift=-1pt]s2.south east)
  {The proof is valid - but it bounds the \emph{even-restricted} constant only.
  Since $C_{6.2}\le C^{\mathrm{even}}_{y.2}$, this is not a bound on $C_{6.2}$: the
  even restriction discards admissible asymmetric~$f$.};
 
\node[theorybubble, anchor=north east] (b3) at ([yshift=-3mm]b2.south east)
  {Concretely, $1.344$ sits below the unrestricted target $C_{6.2}\ge1.28$.
  Drop the even restriction next - e.g.\ a test functional on the autocorrelation
  $F*\check F$, whose Fourier data is nonnegative for \emph{all} admissible~$f$.};
 
\node[speaker, anchor=north west] (s4) at ([yshift=-5mm]b3.south-|L) {CODING AGENT};
\node[codebubble, anchor=north west] (b4) at ([yshift=-1pt]s4.south west)
  {Iteration~v14 removes the even restriction, giving the first rigorous bound on unrestricted $C_{6.2}$.};
 
\end{tikzpicture}
\caption{A representative exchange from the loop (condensed). The coding agent
restricts to even $f$ to make progress; the theory agent verifies the proof but
flags that the even restriction is a genuine loss of generality, so the bound does
not apply to $C_{6.2}$, and redirects the search. Iteration~v14 follows this
recommendation and obtains the first rigorous bound on the unrestricted constant.}
\label{fig:verifier-example}
\end{figure}

\section{Conclusion}
\label{sec:conclusion}
 
We presented a human-in-the-loop framework that applies the autoresearch paradigm
to a target it has not previously been used for: the discovery of convex
relaxations that certify \emph{lower} bounds on sharp-constant inequalities. Where
prior agent-based work improves such constants by constructing extremal examples,
which bound the constant from above, our agents search for valid constraints that
tighten a convex relaxation and thereby raise a certified lower bound. A coding
agent proposes and implements constraints, a theory agent checks their validity
and searches for counterexamples, and a human reviews and compactifies the final
program; the reported bound is then certified by an explicit dual-feasible point
verified in exact arithmetic. On two autocorrelation inequalities of
\cite{tao2025alphaevolve}, this yielded improved certified lower bounds.
 
Although we demonstrate the framework on autocorrelation inequalities, nothing in
it is specific to these problems, or even to sharp-constant inequalities. The same
loop applies wherever a lower bound on a nonconvex optimum can be expressed as a
convex relaxation that is tightened by adjoining valid constraints and certified by
duality-a setting that recurs across optimization, combinatorics, and control.
The agents need only the ability to propose candidate constraints, argue their
validity, and implement the resulting convex program; the certification step is
problem-agnostic. Mapping out the class of nonconvex problems on which this is
effective is, in our view, the most promising direction the method opens.
 
Two limitations also indicate where the framework can be strengthened. The
validity of each constraint currently rests on a mathematical argument confirmed by
human review rather than formally verified. Recent progress on automated proof
generation in Lean-both agentic prover-repair pipelines~\citep{apollo2025} and
autoformalization of abstract results in convex
optimization~\citep{sita2025}-suggests that this step could be machine-checked
end to end, removing the residual reliance on human judgment and making the loop
fully autonomous; coding agents have already been used to autoformalize substantial
bodies of mathematics in Lean~\citep{urban2026topology}. Folding the final review
and compactification into the loop is the corresponding step on the discovery side.
Together, these would turn the present human-in-the-loop framework into an
end-to-end automated and formally verified one.

\section*{Acknowledgements}
This work was supported in part by the National Science Foundation (NSF) CAREER Award under Grant CCF-2236829, in part by the National Institutes of Health under Grant 1R01AG08950901A1, in part by the Office of Naval Research under Grant N00014-24-1-2164, and in part by the Defense Advanced Research Projects Agency under Grant HR00112490441.

\bibliography{references}

@article{tao2025alphaevolve,
  title={Mathematical exploration and discovery at scale},
  author={Georgiev, Bogdan and G{\'o}mez-Serrano, Javier and Tao, Terence and Wagner, Adam Zsolt},
  journal={arXiv preprint arXiv:2511.02864},
  year={2025}
}

@article{firstproof2026,
  title={First Proof},
  author={Abouzaid, Mohammed and Blumberg, Andrew J and Hairer, Martin and Kileel, Joe and Kolda, Tamara G and Nelson, Paul D and Spielman, Daniel and Srivastava, Nikhil and Ward, Rachel and Weinberger, Shmuel and others},
  journal={arXiv preprint arXiv:2602.05192},
  year={2026}
}

@misc{karpathy2026autoresearch,
  author       = {Karpathy, Andrej},
  title        = {autoresearch: {AI} agents running research on single-{GPU}
                  nanochat training automatically},
  year         = {2026},
  howpublished = {\url{https://github.com/karpathy/autoresearch}},
  note         = {GitHub repository, released March 2026; accessed 2026-05-21}
}

@article{robeyns2025selfimproving,
  title={A self-improving coding agent},
  author={Robeyns, Maxime and Szummer, Martin and Aitchison, Laurence},
  journal={arXiv preprint arXiv:2504.15228},
  year={2025}
}

@article{romera2024mathematical,
  title={Mathematical discoveries from program search with large language models},
  author={Romera-Paredes, Bernardino and Barekatain, Mohammadamin and Novikov, Alexander and Balog, Matej and Kumar, M Pawan and Dupont, Emilien and Ruiz, Francisco JR and Ellenberg, Jordan S and Wang, Pengming and Fawzi, Omar and others},
  journal={Nature},
  volume={625},
  number={7995},
  pages={468-475},
  year={2024},
  publisher={Nature Publishing Group UK London}
}

@article{openai2025science,
  title={Early science acceleration experiments with GPT-5},
  author={Bubeck, S{\'e}bastien and Coester, Christian and Eldan, Ronen and Gowers, Timothy and Lee, Yin Tat and Lupsasca, Alexandru and Sawhney, Mehtaab and Scherrer, Robert and Sellke, Mark and Spears, Brian K and others},
  journal={arXiv preprint arXiv:2511.16072},
  year={2025}
}

@article{white2022erd,
  title={Erdos' minimum overlap problem},
  author={White, Ethan Patrick},
  journal={arXiv preprint arXiv:2201.05704},
  year={2022}
}

@article{tttdiscover2026,
  title={Learning to discover at test time},
  author={Yuksekgonul, Mert and Koceja, Daniel and Li, Xinhao and Bianchi, Federico and McCaleb, Jed and Wang, Xiaolong and Kautz, Jan and Choi, Yejin and Zou, James and Guestrin, Carlos and others},
  journal={arXiv preprint arXiv:2601.16175},
  year={2026}
}

@article{apollo2025,
  title={Apollo: Automated llm and lean collaboration for advanced formal reasoning},
  author={Ospanov, Azim and Farnia, Farzan and Mohit, Roozbeh},
  journal={Advances in Neural Information Processing Systems},
  volume={38},
  pages={41599-41633},
  year={2026}
}

@inproceedings{sita2025,
  title={SITA: A Framework for Structure-to-Instance Theorem Autoformalization},
  author={Li, Chenyi and Ma, Wanli and Wang, Zichen and Wen, Zaiwen},
  booktitle={Proceedings of the AAAI Conference on Artificial Intelligence},
  volume={40},
  number={23},
  pages={19224-19232},
  year={2026}
}

@article{urban2026topology,
  title={130k Lines of Formal Topology in Two Weeks: Simple and Cheap Autoformalization for Everyone?},
  author={Urban, Josef},
  journal={arXiv preprint arXiv:2601.03298},
  year={2026}
}

@article{tsoukalas2026advancing,
  title={Advancing Mathematics Research with AI-Driven Formal Proof Search},
  author={Tsoukalas, George and Kovsharov, Anton and Shirobokov, Sergey and Surina, Anja and Firsching, Moritz and B{\'e}rczi, Gergely and Ruiz, Francisco JR and Suggala, Arun and Wagner, Adam Zsolt and Wieser, Eric and others},
  journal={arXiv preprint arXiv:2605.22763},
  year={2026}
}

@article{huang2025winning,
  title={Winning gold at imo 2025 with a model-agnostic verification-and-refinement pipeline},
  author={Huang, Yichen and Yang, Lin F},
  journal={arXiv preprint arXiv:2507.15855},
  year={2025}
}

@misc{einsteinarena2026,
  title        = {{EinsteinArena}: New State-of-the-Art Bounds for Open Problems},
  author       = {{Together AI}},
  year         = {2026},
  howpublished = {\url{https://github.com/togethercomputer/EinsteinArena-new-SOTA}},
  note         = {GitHub repository; new SOTA bounds for open problems in combinatorics and harmonic analysis obtained by AI agents. Accessed 2026-05-25}
}

@article{novikov2025alphaevolve,
  title={Alphaevolve: A coding agent for scientific and algorithmic discovery},
  author={Novikov, Alexander and V{\~u}, Ng{\^a}n and Eisenberger, Marvin and Dupont, Emilien and Huang, Po-Sen and Wagner, Adam Zsolt and Shirobokov, Sergey and Kozlovskii, Borislav and Ruiz, Francisco JR and Mehrabian, Abbas and others},
  journal={arXiv preprint arXiv:2506.13131},
  year={2025}
}

@inproceedings{jimenez2024swe,
  title={Swe-bench: Can language models resolve real-world github issues?},
  author={Jimenez, Carlos E and Yang, John and Wettig, Alexander and Yao, Shunyu and Pei, Kexin and Press, Ofir and Narasimhan, Karthik},
  booktitle={International Conference on Learning Representations},
  volume={2024},
  pages={54107--54157},
  year={2024}
}

@inproceedings{jain2025livecodebench,
  title={Livecodebench: Holistic and contamination free evaluation of large language models for code},
  author={Jain, Naman and Gu, Alex and Li, Wen-Ding and Yan, Fanjia and Zhang, Tianjun and Wang, Sida and Solar-Lezama, Armando and Sen, Koushik and Stoica, Ion},
  booktitle={International Conference on Learning Representations},
  volume={2025},
  pages={58791--58831},
  year={2025}
}

@article{zheng2026livecodebench,
  title={Livecodebench pro: How do olympiad medalists judge llms in competitive programming?},
  author={Zheng, Zihan and Cheng, Zerui and Shen, Zeyu and Zhou, Shang and Liu, Kaiyuan and He, Hansen and Li, Dongruixuan and Wei, Stanley and Hao, Hangyi and Yao, Jianzhu and others},
  journal={Advances in Neural Information Processing Systems},
  volume={38},
  year={2026}
}

@article{cloninger2017suprema,
  title={On suprema of autoconvolutions with an application to Sidon sets},
  author={Cloninger, Alexander and Steinerberger, Stefan},
  journal={Proceedings of the American Mathematical Society},
  volume={145},
  number={8},
  pages={3191--3200},
  year={2017}
}

@article{martin2009supremum,
  title={The supremum of autoconvolutions, with applications to additive number theory},
  author={Martin, Greg and O’Bryant, Kevin},
  journal={Illinois Journal of Mathematics},
  volume={53},
  number={1},
  pages={219--235},
  year={2009},
  publisher={Duke University Press}
}

@misc{optimization-constants-repo,
  title = {Optimization Constants in Mathematics},
  author = {Davis, Damek and Ivanisvili, Paata and Tao, Terence and contributors},
  year = {2026},
  howpublished = {GitHub repository},
  url = {https://github.com/teorth/optimizationproblems}
}

@article{erdHos1955some,
  title={Some remarks on number theory},
  author={Erd{\H{o}}s, Paul},
  journal={Riveon Lematematika},
  volume={9},
  pages={45--48},
  year={1955}
}

@inproceedings{moser1966overlap,
  title={On the overlap of a function with the translation of its complement},
  author={Moser, L and Murdeshwar, MG},
  booktitle={Colloquium Mathematicum},
  volume={15},
  number={1},
  pages={93--97},
  year={1966},
  organization={Institute of Mathematics Polish Academy of Sciences}
}

@article{haugland1996advances,
  title={Advances in the minimum overlap problem},
  author={Haugland, Jan Kristian},
  journal={journal of number theory},
  volume={58},
  number={1},
  pages={71--78},
  year={1996},
  publisher={Elsevier}
}

@article{haugland2016minimum,
  title={The minimum overlap problem revisited},
  author={Haugland, Jan Kristian},
  journal={arXiv preprint arXiv:1609.08000},
  year={2016}
}

@article{matolcsi2010improved,
  title={Improved bounds on the supremum of autoconvolutions},
  author={Matolcsi, M{\'a}t{\'e} and Vinuesa, Carlos},
  journal={Journal of mathematical analysis and applications},
  volume={372},
  number={2},
  pages={439--447},
  year={2010},
  publisher={Elsevier}
}

@article{ye2026evaluation,
  title={Evaluation-driven Scaling for Scientific Discovery},
  author={Ye, Haotian and Lin, Haowei and Tang, Jingyi and Luo, Yizhen and Yang, Caiyin and Su, Chang and Thapa, Rahul and Yang, Rui and Liu, Ruihua and Li, Zeyu and others},
  journal={arXiv preprint arXiv:2604.19341},
  year={2026}
}

@article{wang2025thetaevolve,
  title={Thetaevolve: Test-time learning on open problems},
  author={Wang, Yiping and Su, Shao-Rong and Zeng, Zhiyuan and Xu, Eva and Ren, Liliang and Yang, Xinyu and Huang, Zeyi and He, Xuehai and Ma, Luyao and Peng, Baolin and others},
  journal={arXiv preprint arXiv:2511.23473},
  year={2025}
}

@article{diamond2016cvxpy,
  title={CVXPY: A Python-embedded modeling language for convex optimization},
  author={Diamond, Steven and Boyd, Stephen},
  journal={Journal of Machine Learning Research},
  volume={17},
  number={83},
  pages={1--5},
  year={2016}
}
\bibliographystyle{icml2026}

\newpage
\appendix
\onecolumn
\section{Additional constraint of the convex problem of Erd\H{o}s minimum overlap problem}
\label{AppxA}

Our program for the Erd\H{o}s minimum overlap constant $C_{6.5}$ takes White's convex
program~\citep{white2022erd} as its base and adds two positive-semidefinite constraints discovered by the loop. We treat White's program as given and refer the
reader to \citet{white2022erd} for its derivation; this appendix justifies only the added constraints. 
 
Let $T_f$ be the Hermitian Toeplitz matrix built from the trigonometric moments of
$f$,
\begin{equation}
(T_f)_{kl} \;=\; \tfrac14\,\delta_{kl} \;+\; \tfrac12\bigl(a_{|k-l|} - \mathrm{i}\,\mathrm{sgn}(k-l)\,b_{|k-l|}\bigr),
\label{eq:Tf}
\end{equation}
where $a_m,b_m$ are the cosine and sine Fourier coefficients of $f$ (so that the
$(k,l)$ entry depends only on the lag $k-l$). The loop added the pair
\begin{equation}
T_f \succeq 0 \qquad\text{and}\qquad I - T_f \succeq 0.
\label{eq:bochner-pair}
\end{equation}
We show both are valid necessary conditions, i.e.\ that every admissible $f$
produces moments satisfying \eqref{eq:bochner-pair}; adding them therefore tightens
the relaxation without excluding any admissible $f$.
 
\begin{proposition}\label{prop:bochner}
For every admissible $f$, the matrix $T_f$ of \eqref{eq:Tf} satisfies
$T_f\succeq0$ and $I-T_f\succeq0$.
\end{proposition}
 
\begin{proof}
By construction $T_f$ is the moment (Toeplitz) matrix of $f$: its entries are the
Fourier coefficients of $f$ arranged by lag, so for any complex vector $z$,
\[
z^{*}T_f z \;=\; \int_{-1}^{1} f(x)\,\Bigl|\sum_k z_k e^{\mathrm{i}k\pi x}\Bigr|^{2}\,dx .
\]
Since $f\ge0$, the integrand is nonnegative, hence $z^{*}T_f z\ge0$ for all $z$ and
$T_f\succeq0$. This is Bochner's theorem: a sequence is the moment sequence of a
nonnegative measure iff its Toeplitz form is positive semidefinite.
 
For the second condition, apply the same argument to the complement $g=1-f$. Because
$f\le1$ we have $g\ge0$, so $g$ is also a nonnegative function and its moment matrix
$T_g$ satisfies $T_g\succeq0$. The constant function $1$ has moment matrix equal to
the identity $I$ (its only nonzero Fourier coefficient is the constant term, which
contributes the diagonal), and the moment map is linear, so $T_g = T_{1} - T_f = I -
T_f$. Hence $I-T_f = T_g\succeq0$.
\end{proof}
 
\section{Derivation of the convex problem of the First Autocorrelation Inequality}
\label{AppxB}

 
This section derives the convex relaxation for the sharp constant $C_{6.2}$ of
\texttt{Problem\_6.2.md} and proves rigorously that its optimum $\Omega^{*}$ is a
valid lower bound, $C_{6.2}\ge\Omega^{*}$. The development follows the same
template as White's program for the Erd\H{o}s overlap problem
(Section~1): identify the Fourier/moment data of an admissible $f$, prove a list
of \emph{necessary} linear and semidefinite conditions on that data, and observe
that the relaxed feasible set therefore contains every admissible $f$.
 
\subsection{Reduction and normalisation}
 
We seek a lower bound for
\[
C_{6.2} = \inf_{f\ge 0}\;
\frac{\max_{-1/2\le t\le 1/2}\int_{\Real} f(t-x)f(x)\,dx}
     {\big(\int_{-1/4}^{1/4} f(x)\,dx\big)^2}.
\]
 
\begin{lemma}[Support reduction]\label{lem:support}
It suffices to bound the infimum over $f\ge 0$ with $\operatorname{supp} f \subseteq
[-1/4,1/4]$. Concretely, if for all $g\ge 0$ with $\operatorname{supp} g\subseteq
[-1/4,1/4]$ we have $\max_{\abs t\le 1/2}(g*g)(t)\ge C\big(\int_{-1/4}^{1/4} g\big)^2$,
then the same inequality holds for all $f \ge 0$.
\end{lemma}
 
\begin{proof}
Given $f\ge 0$, set $g = f\cdot\mathbf 1_{[-1/4,1/4]}$, so $0\le g\le f$ and
$\operatorname{supp} g\subseteq[-1/4,1/4]$. The denominators agree:
$\int_{-1/4}^{1/4} g = \int_{-1/4}^{1/4} f$. For the numerator, $g\le f$ and both
are nonnegative, so $(g*g)(t)=\int g(t-x)g(x)\,dx \le \int f(t-x)f(x)\,dx=(f*f)(t)$
pointwise; taking $\max_{\abs t\le 1/2}$ preserves the inequality. Hence
$\max(f*f)\ge\max(g*g)\ge C(\int g)^2 = C(\int f)^2$.
\end{proof}
 
By homogeneity we normalise $\int_{-1/4}^{1/4} f = 1$. Writing $g = f*f$ and
$\Omega = \max_{\abs t\le1/2}(g)(t) = \norm{f*f}_\infty$, the reduced problem is
\begin{equation}\label{eq:62-S}
S \;=\; \inf_{f}\;\Big\{\, \norm{f*f}_\infty \;:\; f\ge 0,\ \int_{-1/4}^{1/4} f = 1,\ \operatorname{supp} f\subseteq[-\tfrac 14,\tfrac14]\,\Big\},
\qquad C_{6.2} = S .
\end{equation}
 
\subsection{Fourier and moment data}
 
For an admissible $f$ define the real Fourier coefficients
\[
a_k = \int_{-1/4}^{1/4} f(x)\cos(2\pi k x)\,dx,
\qquad
b_k = \int_{-1/4}^{1/4} f(x)\sin(2\pi k x)\,dx,
\qquad k = 0,1,\dots,K,
\]
so that $\hat f(k) = a_k - i b_k$ and $\hat f(-k) = a_k + i b_k$, with
$a_0 = \int f = 1$ and $b_0 = 0$. Because $g = f*f$ has $\hat g(k) = \hat f(k)^2$,
the autocorrelation's spectrum is determined by $(a_k,b_k)$.
 
We also discretise the localisation window into $2N$ equal cells
$I_j = \big[-\tfrac14 + \tfrac{j}{4N},\, -\tfrac14 + \tfrac{j+1}{4N}\big]$,
$j = 0,1,\dots,2N-1$, and set the local masses
\[
p_j = \int_{I_j} f(x)\,dx \;\ge\; 0 .
\]
 
\subsection{Necessary conditions}
 
\paragraph{(1) Mass.} Since the cells partition $[-1/4,1/4]$ and $\int_{-1/4}^{1/4}f=1$,
\begin{equation}\label{eq:62-mass}
\sum_{j=0}^{2N-1} p_j = 1, \qquad p_j \ge 0 .
\end{equation}
 
\paragraph{(2) Fourier coefficients from local masses.}
For each cell define the trigonometric extrema
\[
C_{kj}^{-} = \min_{x\in I_j}\cos(2\pi k x),\quad
C_{kj}^{+} = \max_{x\in I_j}\cos(2\pi k x),\quad
S_{kj}^{-} = \min_{x\in I_j}\sin(2\pi k x),\quad
S_{kj}^{+} = \max_{x\in I_j}\sin(2\pi k x).
\]
 
\begin{lemma}[Coefficient envelopes]\label{lem:env}
For every admissible $f$ and every $k$,
\begin{equation}\label{eq:62-env}
\sum_{j=0}^{2N-1} C_{kj}^{-} p_j \;\le\; a_k \;\le\; \sum_{j=0}^{2N-1} C_{kj}^{+} p_j,
\qquad
\sum_{j=0}^{2N-1} S_{kj}^{-} p_j \;\le\; b_k \;\le\; \sum_{j=0}^{2N-1} S_{kj}^{+} p_j .
\end{equation}
\end{lemma}
 
\begin{proof}
Since $f\ge 0$ and $\operatorname{supp} f$ is partitioned by the $I_j$,
\[
a_k = \int_{-1/4}^{1/4} f(x)\cos(2\pi kx)\,dx
    = \sum_{j} \int_{I_j} f(x)\cos(2\pi kx)\,dx .
\]
On each cell, $C_{kj}^- \le \cos(2\pi kx)\le C_{kj}^+$ for all $x\in I_j$, and
$\int_{I_j} f = p_j \ge 0$. Multiplying the pointwise bounds by the nonnegative
density $f$ and integrating gives $C_{kj}^- p_j \le \int_{I_j} f\cos(2\pi kx)\,dx
\le C_{kj}^+ p_j$. Summing over $j$ yields the bound on $a_k$; the argument for
$b_k$ is identical with $\sin$.
\end{proof}
 
\paragraph{(3) Moment matrix.}
Let $y = (1, a_1,\dots,a_K, b_1,\dots,b_K)^\top \in \Real^{2K+1}$ and
$M = y\,y^\top \in \Real^{(2K+1)\times(2K+1)}$, indexed $0,1,\dots,2K$.
 
\begin{lemma}[Moment matrix is feasible and rank one]\label{lem:M}
The matrix $M = yy^\top$ satisfies $M \PSD$ and
\[
M_{0,0} = 1,\quad M_{0,k} = a_k,\quad M_{0,K+k} = b_k,\quad
M_{k,k} = a_k^2,\quad M_{K+k,K+k} = b_k^2 \quad (1\le k\le K).
\]
\end{lemma}
 
\begin{proof}
$M = yy^\top$ is a Gram matrix of a single vector, hence PSD: $u^\top M u =
(y^\top u)^2 \ge 0$ for all $u$. The listed entries are the corresponding
products $y_i y_j$: $y_0 = 1$, $y_k = a_k$, $y_{K+k} = b_k$, so $M_{0,0}=1$,
$M_{0,k}=a_k$, $M_{0,K+k}=b_k$, $M_{k,k}=a_k^2$, $M_{K+k,K+k}=b_k^2$.
\end{proof}
 
In the relaxation we keep only the convex consequences: $M\PSD$, the affine
entry constraints of Lemma~\ref{lem:M} linking the first row/diagonal to
$(a,b)$, and the envelopes \eqref{eq:62-env}. We do \emph{not} impose
$\operatorname{rank} M = 1$ (nonconvex). Every admissible $f$ produces a feasible
$M$, so no admissible point is lost.
 
\paragraph{(4) The objective bound via Parseval.}
 
\begin{lemma}[Spectral lower bound on $\Omega$]\label{lem:parseval}
Let $w_k = 1-\tfrac{k}{K+1}\in[0,1)$ be the Fej\'er weights. For admissible $f$
with $\int_{-1/4}^{1/4}f = 1$, the sharp identity is
\begin{equation}\label{eq:62-omega-sharp}
\Omega \;=\; \norm{f*f}_\infty \;\ge\; 1 + 2\sum_{k=1}^{\infty} \big(a_k^2 + b_k^2\big)^2,
\end{equation}
and consequently the weighted, truncated relaxation used by the program holds:
\begin{equation}\label{eq:62-omega-code}
\Omega \;\ge\; 1 + 2\sum_{k=1}^{K} w_k\,v_k,
\qquad\text{where}\quad v_k \ge \big(M_{k,k}+M_{K+k,K+k}\big)^2 = (a_k^2+b_k^2)^2 .
\end{equation}
\end{lemma}
 
\begin{proof}
Let $g = f*f \ge 0$ (nonnegative as a self-convolution of $f\ge0$). Its total
mass is $\int_{\Real} g = \big(\int f\big)^2 = 1$. Therefore
$\Omega = \norm g_\infty = \norm g_\infty\int g \ge \int g^2$, since
$\int g^2 \le \norm g_\infty\int g$ for $g\ge0$. By Parseval and
$\hat g(m) = \hat f(m)^2$,
\[
\int g^2 = \sum_{m\in\Integer} \abs{\hat g(m)}^2 = \sum_{m\in\Integer}\abs{\hat f(m)}^4
= 1 + 2\sum_{k=1}^{\infty}(a_k^2+b_k^2)^2,
\]
using $\hat f(0)=1$ and $\abs{\hat f(\pm k)}^2 = a_k^2+b_k^2$. This proves
\eqref{eq:62-omega-sharp}. For \eqref{eq:62-omega-code}: every term
$(a_k^2+b_k^2)^2\ge0$ and $0\le w_k\le1$, so
\[
1 + 2\sum_{k=1}^{K} w_k (a_k^2+b_k^2)^2
\;\le\; 1 + 2\sum_{k=1}^{\infty}(a_k^2+b_k^2)^2 \;\le\; \Omega .
\]
With $a_k^2 = M_{k,k}$, $b_k^2 = M_{K+k,K+k}$ (Lemma~\ref{lem:M}) and the epigraph
variable $v_k\ge(M_{k,k}+M_{K+k,K+k})^2$, the left side is exactly the program's
constraint. The weights and truncation only \emph{weaken} the bound, so it
remains a valid lower bound on $\Omega$.
\end{proof}
 
\subsection{The three positive-semidefinite blocks}
 
The program contains three Hermitian PSD blocks, each of size $K+1$ and each the
real embedding of a Hermitian Toeplitz matrix. We verify all three from the
quadratic-form/Bochner principle: a Hermitian Toeplitz matrix $T=(\tau_{i-j})$ is
positive semidefinite whenever its symbol $\tau_k = \int w(x) e^{-2\pi i kx}\,dx$
arises from a \emph{nonnegative} weight $w$, because then for all $z\in\Complex^{K+1}$
\begin{equation}\label{eq:bochner-gen}
z^{*} T z = \sum_{i,j} \bar z_i\, \tau_{i-j}\, z_j
= \int w(x)\,\Big|\sum_{j=0}^{K} z_j e^{2\pi i jx}\Big|^2 dx \ge 0 .
\end{equation}
The three blocks differ only in the choice of nonnegative weight $w$.
 
\paragraph{(5a) Plain Bochner--Toeplitz on $f$.}
Recall $\hat f(k) = a_k - i b_k$ (with $\hat f(-k) = \overline{\hat f(k)}$), and
form the Hermitian Toeplitz $T_f$ with $(T_f)_{ij} = \hat f(i-j)$, i.e. real part
$a_{\abs{i-j}}$ and imaginary part $-b_{\abs{i-j}}$ for $i>j$, $+b_{\abs{i-j}}$ for
$i<j$, diagonal $a_0=1$. In the program this is the \texttt{big} block.
 
\begin{lemma}[Plain Bochner block]\label{lem:bochner-f}
$T_f \PSD$.
\end{lemma}
 
\begin{proof}
Apply \eqref{eq:bochner-gen} with weight $w = f \ge 0$: the symbol is
$\tau_k = \int_{-1/4}^{1/4} f(x) e^{-2\pi i kx}\,dx = \hat f(k) = a_k - i b_k$,
which are precisely the entries of $T_f$. Hence
$z^* T_f z = \int_{-1/4}^{1/4} f(x)\,\big|\sum_j z_j e^{2\pi i jx}\big|^2 dx \ge 0$
since $f\ge0$.
\end{proof}
 
\paragraph{(5b) Localised Bochner with $h(x)=\cos(2\pi x)$.}
Form the Hermitian Toeplitz $T_\nu$ with symbol $\nu_k$, where, writing
$\operatorname{sgn}$ for the sign function,
\[
\operatorname{Re}\nu_k = \tfrac12\big(a_{\abs{k-1}}+a_{\abs{k+1}}\big),
\qquad
\operatorname{Im}\nu_k = -\tfrac12\big(\operatorname{sgn}(k-1)\,b_{\abs{k-1}}+\operatorname{sgn}(k+1)\,b_{\abs{k+1}}\big).
\]
In the program this is the \texttt{big\_nu} block, assembled from \texttt{nu\_re} and \texttt{nu\_im}.
 
\begin{lemma}[Localised Bochner block]\label{lem:bochner-loc}
Define $\nu_k$ by the affine maps
$\operatorname{Re}\nu_k = \tfrac12(a_{\abs{k-1}}+a_{\abs{k+1}})$,
$\operatorname{Im}\nu_k = -\tfrac12(\operatorname{sgn}(k-1)\,b_{\abs{k-1}}+\operatorname{sgn}(k+1)\,b_{\abs{k+1}})$,
and let $T_\nu = (\nu_{i-j})_{i,j=0}^{K}$. Then:
\begin{enumerate}[label=\textup{(\roman*)},leftmargin=2.0em]
\item $\nu_k$ equals the $k$-th Fourier coefficient of $x\mapsto\cos(2\pi x)f(x)$;
\item $\nu_{-k} = \overline{\nu_k}$, so $T_\nu$ is Hermitian Toeplitz; equivalently
$\operatorname{Re}T_\nu$ is symmetric and $\operatorname{Im}T_\nu$ is
skew-symmetric, which is exactly the structure required for the real embedding
$\big[\begin{smallmatrix}\operatorname{Re}T_\nu & -\operatorname{Im}T_\nu\\ \operatorname{Im}T_\nu & \operatorname{Re}T_\nu\end{smallmatrix}\big]$ to represent $T_\nu$;
\item the program imposes the single semidefinite constraint $T_\nu\PSD$ (no
trace or band-sum equalities are attached to this block), and $T_\nu\PSD$ holds
for every admissible $f$.
\end{enumerate}
\end{lemma}
 
\begin{proof}
\emph{(i)} Since $\cos(2\pi x) = \tfrac12(e^{2\pi i x}+e^{-2\pi i x})$,
multiplication by $\cos(2\pi x)$ shifts Fourier coefficients by $\pm1$:
$\widehat{\cos(2\pi\cdot)f}(k) = \tfrac12(\hat f(k-1)+\hat f(k+1))$. Using
$\hat f(m) = a_{\abs m} - i\,\operatorname{sgn}(m)\,b_{\abs m}$ for all
$m\in\Integer$ (with $\hat f(0)=a_0$ real), the real part of
$\tfrac12(\hat f(k-1)+\hat f(k+1))$ is $\tfrac12(a_{\abs{k-1}}+a_{\abs{k+1}})$ and
the imaginary part is
$-\tfrac12(\operatorname{sgn}(k-1)b_{\abs{k-1}}+\operatorname{sgn}(k+1)b_{\abs{k+1}})$,
matching $\nu_k$ exactly.
 
\emph{(ii)} From (i), $\nu_{-k}=\widehat{\cos(2\pi\cdot)f}(-k)=\overline{\widehat{\cos(2\pi\cdot)f}(k)}=\overline{\nu_k}$
because $\cos(2\pi x)f(x)$ is real-valued. Hence
$(T_\nu)_{ji}=\nu_{j-i}=\nu_{-(i-j)}=\overline{\nu_{i-j}}=\overline{(T_\nu)_{ij}}$,
so $T_\nu$ is Hermitian; writing $T_\nu = \operatorname{Re}T_\nu + i\operatorname{Im}T_\nu$,
$\operatorname{Re}T_\nu$ is symmetric and $\operatorname{Im}T_\nu$ is
skew-symmetric, which is precisely what makes the real embedding a faithful
representation (its spectrum is that of $T_\nu$ doubled).
 
\emph{(iii)} Apply \eqref{eq:bochner-gen} with weight $w(x)=\cos(2\pi x)f(x)$:
\[
z^* T_\nu z = \int_{-1/4}^{1/4}\cos(2\pi x)\,f(x)\,\Big|\sum_{j=0}^{K} z_j e^{2\pi i jx}\Big|^2 dx .
\]
On $\operatorname{supp} f=[-1/4,1/4]$ we have $2\pi\abs x\le\pi/2$, so
$\cos(2\pi x)\ge0$ throughout the support of $f$; the integrand is therefore
nonnegative and $z^* T_\nu z\ge0$.
\end{proof}
 
\begin{remark}
The localiser $h(x)=\cos(2\pi x)$ is chosen precisely because its first
nonnegativity interval $[-1/4,1/4]$ coincides with $\operatorname{supp} f$. Any
$h\ge0$ on $[-1/4,1/4]$ whose Fourier support is bounded would furnish a valid
localised block; $\cos(2\pi x)$ is the simplest single-frequency choice and gives
the $\pm1$ coefficient shift above.
\end{remark}
 
\paragraph{(5c) Fej\'er--Riesz block $Q$ for the autocorrelation.}
Let $g = f*f$, so $\hat g(k) = \hat f(k)^2 = (a_k - i b_k)^2 =
(a_k^2-b_k^2) - i\,(2a_k b_k)$, and let $F_K$ be the $K$-th Fej\'er kernel with
weights $w_k = 1 - \tfrac{k}{K+1}$, $F_K\ge0$, $\int F_K = 1$. Define $Q$ as the
Hermitian Toeplitz matrix with symbol
\[
q_0 = \Omega - 1, \qquad q_k = -\,w_k\,\hat g(k)\ \ (1\le k\le K),
\]
in the program this is the \texttt{Q} block (the constraints
$\operatorname{Re}\operatorname{tr}Q=\Omega-1$ and the banded real/imaginary
identities).
 
\begin{lemma}[Fej\'er--Riesz block]\label{lem:Q}
Let $z(t)=(1,e^{it},\dots,e^{iKt})^{\!\top}$ and let $Q$ be the Hermitian
$(K+1)\times(K+1)$ Gram matrix of the localised symbol
$h(t)=\Omega-(F_K*g)(t)$, i.e.\ $h(t)=z(t)^{*}Q\,z(t)$. With $M$ the
rank-one moment matrix of Lemma~\ref{lem:M} (so $M_{k,k}=a_k^2$,
$M_{K+k,K+k}=b_k^2$, $M_{k,K+k}=a_kb_k$), the symbol coefficients
$q_k:=\sum_{l-j=k}Q_{jl}$ satisfy
\[
\operatorname{Re}q_k=-w_k\,(M_{k,k}-M_{K+k,K+k}),
\qquad
\operatorname{Im}q_k=+\,2\,w_k\,M_{k,K+k}
\qquad(1\le k\le K),
\]
the $k$-th diagonal sum of $Q$ reproducing the program's banded constraints.
The zeroth constraint is the trace, $\operatorname{tr}Q=q_0=\Omega-1$.
Moreover $Q\PSD$.
\end{lemma}

\begin{proof}
Substituting $\hat g(k)=(a_k^2-b_k^2)-i\,(2a_kb_k)$
into $q_k=-w_k\hat g(k)$ gives
\[
\operatorname{Re}q_k=-w_k(a_k^2-b_k^2)=-w_k(M_{k,k}-M_{K+k,K+k}),
\qquad
\operatorname{Im}q_k=+w_k(2a_kb_k)=2w_kM_{k,K+k}.
\]
Expanding $h(t)=z(t)^{*}Qz(t)=\sum_{j,l}Q_{jl}e^{i(l-j)t}$ and matching the
coefficient of $e^{ikt}$ gives $\hat h(k)=\sum_{l-j=k}Q_{jl}=q_k$, so these
are exactly the banded (diagonal-sum) constraints of the program. For $k=0$,
$\hat g(0)=(\int f)^2=1$ and $w_0=1$, whence
$q_0=\hat h(0)=\Omega-\widehat{F_K*g}(0)=\Omega-1$; since the $0$-th diagonal
is the main diagonal, $\operatorname{tr}Q=q_0=\Omega-1$.

The symbol $h$ is a nonnegative
trigonometric polynomial of degree $K$: indeed $g=f*f\ge0$ and $F_K\ge0$
with $\int F_K=1$, so by Jensen/averaging
$(F_K*g)(t)\le\norm g_\infty=\Omega$, giving $h=\Omega-F_K*g\ge0$, while
$\hat h(k)=w_k\hat g(k)$ vanishes for $\abs k>K$ since $w_k=0$ there. By the
Fej\'er--Riesz theorem there is an analytic polynomial
$p(t)=\sum_{j=0}^{K}c_je^{ijt}$ with
\[
h(t)=\abs{p(t)}^2=z(t)^{*}\,(cc^{*})\,z(t),
\qquad c=(c_0,\dots,c_K)^{\!\top}.
\]
Hence $Q=cc^{*}\PSD$ realises the symbol $h$ with the prescribed diagonal
sums: it is the Fej\'er--Riesz certificate, and $Q\PSD$ as claimed.
\end{proof}
 
\paragraph{Real embedding.}
As in Section~1, each Hermitian block $T = A + iB$ ($A$ symmetric, $B$
skew-symmetric) is realised in CVXPY through the real embedding
$\big[\begin{smallmatrix} A & -B\\ B & A\end{smallmatrix}\big]\PSD$, equivalent to
$T\PSD$ because the embedding's spectrum is that of $T$ with each eigenvalue
doubled. In the program this is the \texttt{cp.bmat([[T\_R,-T\_I],[T\_I,T\_R]])}
construction, used for all three Hermitian blocks.
 
\subsection{The relaxed convex program and its validity}
 
\begin{definition}[Relaxation]\label{def:62-program}
With $g=f*f$, $\hat f(k)=a_k-\mathbf{i}b_k$, Fej\'er weights $w_k=1-\tfrac{k}{K+1}$,
and $\gamma^\pm_{k},\sigma^\pm_{k}$ the per-cell lower/upper envelopes of
$\cos(2\pi kx),\sin(2\pi kx)$, the value $\Omega^{*}$ is the optimum of
\begin{equation}\label{eq:62-program}
\begin{aligned}
\Omega^{*} = \min_{\Omega,\,p,\,a,\,b,\,M,\,Q,\,v}\quad & \Omega\\
\text{s.t.}\quad
& \textstyle\sum_j p_j = 1,\ p_j\ge0 && \text{(mass)}\\
& \gamma^-_{k}\!\cdot p \le a_k \le \gamma^+_{k}\!\cdot p,\quad
  \sigma^-_{k}\!\cdot p \le b_k \le \sigma^+_{k}\!\cdot p && \text{(Lemma \ref{lem:env})}\\
& M \succeq 0,\ M_{0,0}=1,\ M_{0,k}=a_k,\ M_{0,K+k}=b_k && \text{(Lemma \ref{lem:M})}\\
& v_k \ge (M_{k,k}+M_{K+k,K+k})^2,\quad
  \Omega \ge 1 + 2\textstyle\sum_{k=1}^{K} w_k\,v_k && \text{(Lemma \ref{lem:parseval})}\\
& Q \succeq 0,\ \operatorname{tr} Q = \Omega-1 && \text{(Lemma \ref{lem:Q})}\\
& \textstyle\sum_i Q_{i,i+k}
  = \mathbf{i}\,2w_k M_{k,K+k} - w_k\bigl(M_{k,k}-M_{K+k,K+k}\bigr) && \text{(Fej\'er band, Lemma \ref{lem:Q})}\\
& T_f \succeq 0,\quad (T_f)_{ij}=a_{|i-j|}-\mathbf{i}\,\mathrm{sgn}(i-j)\,b_{|i-j|} && \text{(Lemma \ref{lem:bochner-f})}\\
& T_\nu \succeq 0,\quad (T_\nu)_{ij} = \nu_{i-j}, \quad \nu_k=\tfrac12(\hat\mu_{k-1}+\hat\mu_{k+1}) && \text{(Lemma \ref{lem:bochner-loc})}.
\end{aligned}
\end{equation}
The envelope coefficients $\gamma^\pm_{k},\sigma^\pm_{k}$ are computed by the
\emph{analytical} cell extrema of $\cos(2\pi kx),\sin(2\pi kx)$ on each cell (the
\texttt{cos\_extrema\_on\_cell}/\texttt{sin\_extrema\_on\_cell} routines), not a
grid sample. This matters for rigor: a grid minimum over-estimates the true minimum,
which would make $a_k \ge (\text{grid min})\cdot p$ \emph{stricter} than correct and
could exclude an admissible $f$. The analytical extrema (endpoints together with the
interior critical points $x=m/(2k)$ for $\cos$ and $x=(2m+1)/(4k)$ for $\sin$) give
the true cell min/max, so Lemma~\ref{lem:env} holds exactly.
\end{definition}
 
\begin{theorem}[Validity of the relaxation]\label{thm:62}
$C_{6.2} = S \ge \Omega^{*}$.
\end{theorem}
 
\begin{proof}
Let $f$ be any admissible function for the reduced problem~\eqref{eq:62-S}:
$f\ge0$, $\int_{-1/4}^{1/4}f=1$, $\operatorname{supp} f\subseteq[-1/4,1/4]$. We
construct a feasible point of~\eqref{eq:62-program} with objective value
$\norm{f*f}_\infty$.
 
Set $p_j = \int_{I_j} f$, $a_k,b_k$ the Fourier coefficients of $f$,
$y=(1,a,b)$, $M = yy^\top$, $\Omega = \norm{f*f}_\infty$, $v_k=(a_k^2+b_k^2)^2$,
and $T_f, T_\nu, Q$ the three Hermitian Toeplitz blocks built from $f$'s data.
Each constraint of~\eqref{eq:62-program} is then satisfied:
\begin{itemize}[leftmargin=1.4em]
\item the mass constraint by~\eqref{eq:62-mass};
\item the coefficient envelopes by Lemma~\ref{lem:env} (analytical cell extrema);
\item $M\PSD$ and its entry identities by Lemma~\ref{lem:M};
\item the objective inequality $\Omega\ge1+2\sum_k w_k v_k$ by Lemma~\ref{lem:parseval};
\item $T_f\PSD$ by Lemma~\ref{lem:bochner-f} (weight $f\ge0$);
\item $T_\nu\PSD$ by Lemma~\ref{lem:bochner-loc} (weight $\cos(2\pi x)f\ge0$ on $\operatorname{supp} f$);
\item $Q\PSD$ and the Fej\'er band identities by Lemma~\ref{lem:Q} (weight $\Omega-F_K*g\ge0$).
\end{itemize}
Every one of these is a \emph{necessary} condition  -  proved above for an
arbitrary admissible $f$, not merely for an optimiser  -  so the constructed
point is feasible. Its objective is $\Omega = \norm{f*f}_\infty$. Therefore
\[
\Omega^{*} \le \norm{f*f}_\infty .
\]
Taking the infimum over all admissible $f$ gives $\Omega^{*} \le S = C_{6.2}$,
i.e. $\Omega^{*}$ is a valid lower bound.
\end{proof}
 
\section{Floating point precision details}
\label{AppxC}

 
A numerical conic solver returns an approximate primal--dual pair that is in
general neither exactly feasible nor exactly optimal, and is computed in
floating-point arithmetic on data that is itself a floating-point rounding of the
intended (rational/real) problem. To turn the solver output into a
\emph{mathematically rigorous} lower bound we use a certified weak-duality
argument: we verify in directed-rounding interval arithmetic that a candidate
dual point lies in the dual cone, and charge every rounding and data error to a
penalty controlled by an a-priori bound on the primal solution norm.
 
\subsection{Canonical primal and dual}
 
Our certified-duality method relies on the canonicalisation produced by
\texttt{CVXPY} for the SCS standard form. After canonicalisation the primal is
\begin{equation}\label{eq:primal}
\min_{x}\; c^{\top} x \quad\text{s.t.}\quad A x + s = b,\ \ s \in K,
\end{equation}
with the product cone $K = K_{\mathrm{zero}} \times K_{\ge 0} \times
K_{\mathrm{SOC}} \times K_{\mathrm{PSD}}$ in exactly that order. The Lagrangian
dual is
\begin{equation}\label{eq:dual}
\max_{y}\; -b^{\top} y \quad\text{s.t.}\quad A^{\top} y + c = 0,\ \ y \in \Kstar,
\end{equation}
where $\Kstar$ is the dual cone (here $\Kstar=K$ blockwise, except the zero cone
dualises to the free cone). For any dual-feasible $y$, weak duality gives
$-b^{\top} y \le c^{\top} x^{*} = \Omega^{*}$, so a rigorously dual-feasible $y$
certifies the lower bound $-b^\top y$.
 
\subsection{Error model and the master certificate}
 
The stored data $(A,b,c)$ are floating-point; the intended data are
$(\hat A,\hat b,\hat c)$, with rigorous outward-rounded elementwise error bounds
\[
\max_{ij}\abs{A_{ij}-\hat A_{ij}} \le \varepsilon_A,\qquad
\max_i \abs{b_i-\hat b_i} \le \varepsilon_b,\qquad
\max_j \abs{c_j-\hat c_j} \le \varepsilon_c .
\]
We estimate $\varepsilon_A, \varepsilon_b$ and $\varepsilon_c$ by using mpmath arbitrary-precision library and verify that for our problem sizes, $\epsilon_A, \epsilon_b, \epsilon_c \leq 1e-12$.

Now we show how we can get a certified lower bound.
 
\begin{theorem}[Certified lower bound]\label{thm:cert}
Let $y$ be verified to lie in $\Kstar$ (Algorithm~\ref{alg:cone}). Let
$r \ge \norm{A^{\top} y + c}_{\infty}$ be an outward-rounded upper bound on the
floating dual-feasibility residual and $X \ge \norm{x^{*}}_1$ an a-priori bound
(Section~\ref{sec:Xbound}). With
\[
\varepsilon_{\mathrm{tot}} = r + \varepsilon_A \norm y_1 + \varepsilon_c,
\qquad
\Pi = \varepsilon_{\mathrm{tot}}\,X + \varepsilon_b\,\norm y_1,
\]
one has
\[
\Omega^{*} \;\ge\; L \;:=\; (-b^{\top} y)_{\downarrow} - \Pi_{\uparrow},
\]
computed with directed rounding ($\downarrow$ down, $\uparrow$ up).
\end{theorem}
 
\begin{proof}
Let $x^*$ be any optimiser of the intended primal, so $\hat A x^* + \hat s = \hat b$ with $\hat s\in K$, and $\Omega^* = \hat c^\top x^*$. Since $y\in\Kstar$ and
$\hat s\in K$, $\inner{y}{\hat s} \ge 0$, hence
\[
\hat c^\top x^* \ge \hat c^\top x^* - \inner{y}{\hat s}
= \hat c^\top x^* - \inner{y}{\hat b - \hat A x^*}
= -\hat b^\top y + (\hat c + \hat A^\top y)^\top x^* .
\]
Now replace intended data by floating data plus a bounded perturbation. Writing
$\hat A = A + \Delta_A$, $\hat b = b+\Delta_b$, $\hat c = c+\Delta_c$ with the
stated elementwise bounds,
\[
-\hat b^\top y = -b^\top y - \Delta_b^\top y \ge -b^\top y - \varepsilon_b\norm y_1,
\]
and
\[
(\hat c + \hat A^\top y)^\top x^*
= (c + A^\top y)^\top x^* + (\Delta_c + \Delta_A^\top y)^\top x^*
\ge -\big(\norm{c+A^\top y}_\infty + \varepsilon_A\norm y_1 + \varepsilon_c\big)\norm{x^*}_1,
\]
using H\"older's inequality $\abs{w^\top x^*}\le\norm w_\infty\norm{x^*}_1$,
the residual bound $\norm{c+A^\top y}_\infty \le r$, and the elementwise bounds
$\norm{\Delta_A^\top y}_\infty \le \varepsilon_A\norm y_1$,
$\norm{\Delta_c}_\infty \le \varepsilon_c$. Combining and using $\norm{x^*}_1\le X$,
\[
\Omega^* = \hat c^\top x^* \ge -b^\top y - \big(r+\varepsilon_A\norm y_1+\varepsilon_c\big)X - \varepsilon_b\norm y_1
= -b^\top y - \Pi .
\]
Evaluating $-b^\top y$ with downward rounding and $\Pi$ with upward rounding only
decreases the right-hand side, so $\Omega^*\ge L$.
\end{proof}
 
\subsection{A-priori bound on \texorpdfstring{$\norm{x^*}_1$}{the primal norm}}
\label{sec:Xbound}
 
The constant $X$ is computed in the \emph{CVXPY-canonicalised} coordinates, block
by block, with all arithmetic outward-rounded in interval form; the routines
return the upper endpoint. 

\begin{theorem}[Primal bound, $C_{6.5}$]\label{thm:xbound-65}
Let $L=2/N$ and let $\varepsilon$ be the equality tolerance. Every primal-feasible
point $x$ of the overlap program satisfies $\|x\|_1\le X_{6.5}$, where $X_{6.5}$ is
the outward-rounded sum of the block bounds established in the proof.
\end{theorem}

\begin{proof}
We bound each block from its governing constraint.

\emph{Objective.} The objective variable contributes $\Omega\le1$, since
$\Omega\le1$ is imposed directly (an average overlap cannot exceed $1$).

\emph{Mass.} From the normalization $L\sum_j(w_j+v_j)\le1+\varepsilon$ and
$w_j,v_j\ge0$,
\[
\|w\|_1+\|v\|_1=\sum_j(w_j+v_j)\le\frac{1+\varepsilon}{L}=\frac{N(1+\varepsilon)}{2}.
\]

\emph{Fourier coefficients.} The Parseval constraint gives $|c_k|,|d_k|\le2/\pi$, and
there are $T$ of each, so $\|c\|_1+\|d\|_1\le\tfrac{2}{\pi}\cdot2T=\tfrac{4T}{\pi}$.

\emph{Derived moments.} The coefficients $a_m,b_m$ are entries of the
positive-semidefinite moment matrix associated with $f$: $a_m$ and $b_m$ are the real
and imaginary parts of an off-diagonal entry $(T^f)_{ij}$ at lag $|i-j|=m$. For a PSD
matrix the off-diagonal entries are controlled by the diagonal,
\[
|(T^f)_{ij}|\;\le\;\sqrt{(T^f)_{ii}\,(T^f)_{jj}}\;\le\;\tfrac12\bigl((T^f)_{ii}+(T^f)_{jj}\bigr),
\]
and each diagonal entry is a zeroth moment of a measure of total mass at most $2$, so
$(T^f)_{ii}\le 2$ and therefore $|a_m|,|b_m|\le 2$. With at most $2R$ indices for each
of $a$ and $b$,
\[
\|a\|_1\le 2\cdot 2R=4R,\qquad \|b\|_1\le 4R,
\qquad\text{hence}\qquad \|a\|_1+\|b\|_1\le 8R .
\]

\emph{Truncation slacks.} The slacks are bounded in closed form by
\[
|\varepsilon_{2m-1}|\le\frac{2(2m-1)}{\pi\bigl(4-((2m-1)/T)^2\bigr)\sqrt{6T^3}},
\qquad
|\delta_{2m-1}|\le\frac{4}{\pi\bigl(4-((2m-1)/T)^2\bigr)\sqrt{2T}},
\]
for $1\le m\le R$ (the hypothesis $R\le T$ keeps each denominator positive); their
$\ell_1$ contribution is the sum of these closed forms, evaluated in interval
arithmetic.

\emph{Canonicalization auxiliaries.} Conic canonicalization introduces $4R$ square epigraph variables, each bounded by $4$ since $a_m, b_m \in [-2,2]$; two
sum-of-squares epigraphs, each bounded by $4T/\pi^2$ from the Parseval budget; and
$4R+1$ absolute-value slacks, each bounded by the tolerance $\varepsilon$.

Summing all contributions and rounding outward gives $X_{6.5}\ge\|x^\star\|_1$.
\end{proof}

\begin{theorem}[Primal bound, $C_{6.2}$]\label{thm:xbound-62}
Let $\omega=\Omega_{\mathrm{ub}}-1$ with $\Omega_{\mathrm{ub}}=2$. Every
primal-feasible point $x$ of the autocorrelation program satisfies
$\|x\|_1\le X_{6.2}$, where $X_{6.2}$ is the outward-rounded sum of the block bounds
established in the proof.
\end{theorem}

\begin{proof}
The cap $\Omega\le\Omega_{\mathrm{ub}}=2$ holds because the program's optimum lies
below the best known upper bound on $C_{6.2}$, so $\omega=\Omega_{\mathrm{ub}}-1\ge
\Omega-1\ge0$. We bound each variable block in turn.

\emph{Objective and mass.} The objective variable satisfies
$\Omega\le\Omega_{\mathrm{ub}}=2$. The simplex constraint $\sum_j p_j=1$ with
$p_j\ge0$ gives $\|p\|_1=1$.

\emph{Moments.} The envelope constraints
$\gamma^-_k\!\cdot p\le a_k\le\gamma^+_k\!\cdot p$ and the bound
$|\gamma^\pm_k|\le1$ (cosine enclosures) give $|a_k|\le\sum_j p_j=1$ for each of the
$K+2$ indices $k=0,\dots,K+1$, so $\|a\|_1\le K+2$; likewise $\|b\|_1\le K+1$.

\emph{Moment matrix.} We first bound the trace. The energy constraint
$\Omega\ge 1+2\sum_{k=1}^{K} w_k v_k$ together with the cap gives
\[
\sum_{k=1}^{K} w_k v_k \;\le\; \tfrac12(\Omega-1)\;\le\;\tfrac{\omega}{2}.
\]
The slack constraint $v_k\ge(M_{kk}+M_{K+k,K+k})^2$ and the nonnegativity of the
diagonal of the PSD matrix $M$ give $M_{kk}+M_{K+k,K+k}\le\sqrt{v_k}$, so by
Cauchy--Schwarz,
\[
\sum_{k=1}^{K}\bigl(M_{kk}+M_{K+k,K+k}\bigr)
\;\le\;\sum_{k=1}^{K}\sqrt{v_k}
\;\le\;\Bigl(\sum_{k=1}^{K}\tfrac{1}{w_k}\Bigr)^{1/2}
        \Bigl(\sum_{k=1}^{K} w_k v_k\Bigr)^{1/2}
\;\le\;\Bigl(\sum_{k=1}^{K}\tfrac{1}{w_k}\Bigr)^{1/2}\sqrt{\tfrac{\omega}{2}}.
\]
The Fej\'er weights are $w_k=\tfrac{K+1-k}{K+1}$, so $\tfrac{1}{w_k}=
\tfrac{K+1}{K+1-k}\le K+1$ for $1\le k\le K$, whence $\sum_{k=1}^{K}\tfrac{1}{w_k}\le
K(K+1)$. With $M_{00}=1$,
\[
\tr M \;=\; 1+\sum_{k=1}^{K}\bigl(M_{kk}+M_{K+k,K+k}\bigr)
\;\le\; 1+\sqrt{\tfrac{K(K+1)\,\omega}{2}}.
\]
For the vectorized block, $M\succeq0$ gives $|M_{ij}|\le\sqrt{M_{ii}M_{jj}}
\le\tfrac12(M_{ii}+M_{jj})$. The symmetric vectorization has entries $M_{ii}$ on the
diagonal and $\sqrt2\,M_{ij}$ off-diagonal, so each diagonal entry $M_{ii}$ enters the
off-diagonal sum with total weight $\sqrt2\cdot\tfrac12\cdot(\text{number of partners})
\le\tfrac{\sqrt2}{2}\,(2K)$. Hence
\[
\|\svec(M)\|_1
=\sum_i M_{ii}+\sqrt2\!\sum_{i<j}|M_{ij}|
\le \bigl(1+\sqrt2\,K\bigr)\tr M .
\]

\emph{Slacks and epigraph.} From $\sum_k w_k v_k\le\omega/2$ and $w_k\le1$ we obtain
$\|v\|_1=\sum_k v_k$; using $w_k\ge w_K=\tfrac{1}{K+1}$ this is bounded by
$(K+1)\sum_k w_k v_k\le(K+1)\omega/2$. The epigraph auxiliary cannot exceed
$\|v\|_1$ by its defining inequality.

\emph{Hermitian block.} The constraint $\operatorname{tr}Q=\Omega-1=\omega$ together
with $Q\succeq0$ gives $\sum_i Q_{ii}=\omega$, and the same PSD/vectorization
argument as for $M$ yields $\|\herm(Q)\|_1\le\bigl(1+\tfrac{\sqrt2\,K}{2}\bigr)\omega$,
the $\sqrt2$ again being the off-diagonal scaling of the vectorization.

Summing the block bounds - $\Omega_{\mathrm{ub}}$, $1$, $K+2$, $K+1$,
$(1+\sqrt2K)\tr M$ with $\tr M\le1+\sqrt{K(K+1)\omega/2}$, $(K+1)\omega/2$ (twice, for
$v$ and the epigraph), and $(1+\tfrac{\sqrt2K}{2})\omega$ - each rounded outward,
gives $X_{6.2}\ge\|x^\star\|_1$.
\end{proof}

\subsection{Rigorous certificate of dual-cone containment}
 
It remains to verify $y\in\Kstar$. We walk the product cone in SCS order using
directed-rounding interval arithmetic (the \texttt{mpmath} interval type
\texttt{iv}). The nonnegative and second-order cones are self-dual; the PSD cone
is self-dual and is certified by a residual-plus-spectral-bound test: a floating
Cholesky factor $L$ furnishes a rigorous lower bound on $\lambda_{\min}(LL^\top)$
through the exact inverse $B=L^{-1}$, and the discrepancy between $LL^\top$ and
the block is absorbed by a rigorous residual bound, the two combined by Weyl's
inequality. 
\begin{algorithm}[!ht]
\caption{Rigorous dual-cone containment}
\label{alg:cone}
\begin{algorithmic}[1]
\REQUIRE candidate dual $y$; cone dimensions
  $(\dim_{\mathrm{zero}},\dim_{\ge0},\{q_i\},\{n_j\})$; PSD shift $\sigma>0$;
  precision \texttt{dps}
\ENSURE \textsc{ok} (all blocks certified) and a rigorous margin per block
\STATE \textbf{Zero cone:} the dual is free; skip.
\STATE \textbf{Nonnegative cone:} verify $\min_i y_i\ge0$ by floating comparison
  (the float value is itself the witness); margin $\gets\min_i y_i$.
\STATE \textbf{Second-order cones} $y=(t,z)$: interval-enclose
  $\|z\|_2=\sqrt{\sum_k z_k^2}$ and verify $(t-\|z\|_2)_\downarrow>0$;
  margin $\gets(t-\|z\|_2)_\downarrow$.
\STATE \textbf{PSD cones} ($n_j\times n_j$): reshape the SCS $\svec$ block to a
  symmetric $Y$ (off-diagonals divided by $\sqrt2$), then:
\STATE \quad (a) attempt a floating Cholesky of
  $A_{\mathrm{mid}}=\tfrac12(Y+Y^\top)-\sigma I$, retaining the computed lower
  factor $L$; return \textsc{fail} if it does not complete;
\STATE \quad (b) form the exact interval inverse $B=L^{-1}$ by triangular forward
  substitution and $C=B^\top B=(LL^\top)^{-1}$; bound
  $\lambda_{\max}(C)\le\lambda_{\max}^{+}$ by Gershgorin, giving the rigorous bound
  $\lambda_{\min}(LL^\top)\ge 1/\lambda_{\max}^{+}$;
\STATE \quad (c) rigorously enclose the residual $E=(A-\sigma I)-LL^\top$ in
  Frobenius norm over the interval $A_{ij}\in[Y_{ij}-\rho,Y_{ij}+\rho]$, obtaining
  $\|E\|_F^{+}\ge\|E\|_2$;
\STATE \quad (d) certify with margin
  $m=\sigma+1/\lambda_{\max}^{+}-\|E\|_F^{+}$, a rigorous lower bound on
  $\lambda_{\min}(A)$; the block passes iff $m>0$.
\STATE \textbf{return} \textsc{ok} $=$ (all blocks certified), with the per-block margins.
\end{algorithmic}
\end{algorithm}
 
\section{Prompts used for the agent}
\label{AppxD}

 
The pipeline is driven by three role prompts: a problem specification (one per
inequality), a proposal agent (the coding agent that proposes and implements
constraints), and a verifier agent (the theory agent that adjudicates each
proof). The texts below are reproduced verbatim.

\subsection*{\texttt{proposal\_6.2.md} \textnormal{(coding agent, inequality 6.2)}}
\begin{tcolorbox}[breakable,colback=gray!4,colframe=gray!55,fontupper=\small\ttfamily]
\begin{lstlisting}[breaklines=true,
breakatwhitespace=true,
columns=fullflexible,]

Hi Claude, our mission is to find rigorous advancements of a series of Autocorrelation inequalities that has been proposed. Imagine you are an expert in functional analysis, optimization, Fourier analysis, and combinatorics.
 
Your job is to propose a convex-optimization like formulation that may be motivated from the form proposed in White's paper (the paper directory is in 2201.05704v1.pdf) and constantly improve it. Basically you are going to show a lower-bound or an upper-bound that is achieved by relaxing the complicated problem into a decomposed set of constrained convex problems, then optimize each problem. So your main job is to find novel constraints that were overlooked by previous versions of the program - and use it to optimize. See 2201.05704v1.pdf for an insight: Lemma 3,4,5 is used to construct a convex optimiation problem in section 5.
 
The problem that you need to solve is specified in Problem.md. Also in Problem.md there are tips for the problem. Note that we are trying to tackle these problems in a similar approach - but we are not trying to solve only the erdos minimum overlap problem.
 
Your workflow will be the following:
 
0. This is the initial step. See if you have folder v0. If you do, then run the code as a baseline to check that it works. If there is no v0, you can come up with the most naive version of the convex optimization problem that can be meaningful.
 
After this you have an iterative loop.
 
1. Read the most recent version folder. Then, create {nextversion} folder and copy the contents from the previous folder.
2. Come up with a novel proposal that improves upon the previous implementation. The proposal should not be something like a hyperparameter tweak; we will enhance the hyperparameters after the convex program has converged. 
3. After you come up with a proposal, you have to generate a "rigorous" proof that the constraint indeed holds. For an example see Lemma 3,4,5 in the convex optimization problem of White and how it is related to the convex optimization problem in section 5.
4. Store the proof in the current version folder. Name it as rigorousproof.md
5. Now this is the important: you should be idle until the theorist agent reads your proof and gives a verdict. It will be in the same folder, verdict.md. If the verdict is VALID, it means your proposed proof is rigorous and makes sense. If the verdict is INVALID, it means your proposed proof has a flaw.
6. IF the verdict was VALID, implement the constraint in the convex problem.
7. IF the verdict was INVALID, go to 2 and propose a different constraint that makes mathematical sense.
 
*** THINGS TO CONSIDER ***
- A validation should take less than 10 minutes. Don't heavily optimize the parameter space (split the original problem into too many different convex optimization problems) or increase the problem parameters too much. Our job is NOT to improve the bound for now, but to improve the PROGRAM that will eventually lead to better bounds.
 
- MATHEMATICAL RIGOR IS THE MOST IMPORTANT!!! Do not simply assume that you are doing an approximation that is good enough or do not propose things that may or may not be a mathematically valid lower or upper bound.
 
- NEVER STOP: Once the experiment loop has begun (after the initial setup), do NOT pause to ask the human if you should continue. Do NOT ask "should I keep going?" or "is this a good stopping point?". The human might be asleep, or gone from a computer and expects you to continue working indefinitely until you are manually stopped. You are autonomous. If you run out of ideas, think harder - read papers referenced in the code, re-read the in-scope files for new angles, try combining previous near-misses, try more radical architectural changes. The loop runs until the human interrupts you, period.
 
- Especially, I will not initiate a new interaction when the verdict.md is given. So when I say idle I mean don't stop the current session and wait.
\end{lstlisting}
\end{tcolorbox}

\subsection*{\texttt{Problem\_6.2.md} \textnormal{(problem description, inequality 6.2)}}
\begin{tcolorbox}[breakable,colback=green!4,colframe=green!55,fontupper=\small\ttfamily]
\begin{lstlisting}[breaklines=true,
breakatwhitespace=true,
columns=fullflexible,]
We want a lower bound on the sharp constant (C_{6.2}) in the inequality [ \max_{-1/2\le t\le 1/2}\int_{\mathbb R} f(t-x)f(x),dx ;\ge; C_{6.2}\left(\int_{-1/4}^{1/4} f(x),dx\right)^2 ] for all nonnegative functions (f:\mathbb R\to\mathbb R).

Equivalently, (C_{6.2}) is the infimum of [ \frac{\displaystyle \max_{-1/2\le t\le 1/2}\int_{\mathbb R} f(t-x)f(x),dx} {\displaystyle \left(\int_{-1/4}^{1/4} f(x),dx\right)^2} ] over all nonnegative (f) for which the denominator is nonzero.

The numerator is a localized self-convolution: [ (ff)(t)=\int_{\mathbb R} f(t-x)f(x),dx, ] so the problem asks for the smallest possible value of [ \max_{|t|\le 1/2}(ff)(t) ] relative to the square of the mass of (f) on ([-1/4,1/4]).

To prove a valid lower bound (L) for (C_{6.2}), we need to show that every admissible nonnegative (f) satisfies [ \max_{|t|\le 1/2}(f*f)(t) ;\ge; L\left(\int_{-1/4}^{1/4} f(x),dx\right)^2. ]

A natural way to do this is to formulate a constrained optimization problem. The idea is to replace the original infinite-dimensional problem by a tractable problem whose feasible set contains all objects arising from admissible (f), and whose constraints encode necessary properties that every such (f) must satisfy.

This makes sense because of the relaxation principle:

every genuine nonnegative (f) gives a feasible point of the constrained problem,
the constrained problem is designed to be easier to analyze or solve,
and minimizing over a larger feasible set can only decrease the value.
Therefore, if the constrained optimization problem has optimum (L_{\mathrm{relax}}), then automatically [ L_{\mathrm{relax}} \le C_{6.2}. ] So (L_{\mathrm{relax}}) is a valid lower bound for the original sharp constant.

In other words, to obtain a rigorous lower bound, we should construct a tractable optimization problem whose feasible set contains every admissible (f), and which encodes necessary conditions such as

[ f\ge 0, ] together with whatever additional structure we can prove is relevant for the quantities [ \max_{|t|\le 1/2}(f*f)(t) \qquad\text{and}\qquad \int_{-1/4}^{1/4} f(x),dx. ]

For example, one may introduce auxiliary variables representing local mass, convolution values, Fourier coefficients, or moment data, and then impose constraints that every true (f) must satisfy. Solving the resulting constrained optimization problem then produces a mathematically valid lower bound for (C_{6.2}).

So the overall strategy is:

rewrite (C_{6.2}) as an infimum over nonnegative (f),
derive necessary constraints satisfied by every admissible (f),
build a constrained optimization problem from those necessary conditions,
minimize the target quantity over that relaxed feasible set,
use the resulting optimum as a provable lower bound for (C_{6.2}).
The current SOTA lower bound is 1.28.
\end{lstlisting}
\end{tcolorbox}
 
\subsection*{\texttt{proposal\_erdos.md} \textnormal{(coding agent, Erd\H{o}s; differs only in the hard-square tip)}}
\begin{tcolorbox}[breakable,colback=gray!4,colframe=gray!55,fontupper=\small\ttfamily]
\begin{lstlisting}[breaklines=true,
breakatwhitespace=true,
columns=fullflexible,]
Hi Claude, our mission is to find rigorous advancements of a series of Autocorrelation inequalities that has been proposed. Imagine you are an expert in functional analysis, optimization, Fourier analysis, and combinatorics.
 
Your job is to propose a convex-optimization like formulation that may be motivated from the form proposed in White's paper (the paper directory is in 2201.05704v1.pdf) and constantly improve it. Basically you are going to show a lower-bound or an upper-bound that is achieved by relaxing the complicated problem into a decomposed set of constrained convex problems, then optimize each problem. So your main job is to find novel constraints that were overlooked by previous versions of the program - and use it to optimize. See 2201.05704v1.pdf for an insight: Lemma 3,4,5 is used to construct a convex optimiation problem in section 5.
 
The problem that you need to solve is specified in Problem.md. Also in Problem.md there are tips for the problem. Note that we are trying to tackle these problems in a similar approach - but we are not trying to solve only the erdos minimum overlap problem.
 
Your workflow will be the following:
 
0. This is the initial step. See if you have folder v0. If you do, then run the code as a baseline to check that it works. If there is no v0, you can come up with the most naive version of the convex optimization problem that can be meaningful.
 
After this you have an iterative loop.
 
1. Read the most recent version folder. Then, create {nextversion} folder and copy the contents from the previous folder.
2. Come up with a novel proposal that improves upon the previous implementation. The proposal should not be something like a hyperparameter tweak; we will enhance the hyperparameters after the convex program has converged. 
3. After you come up with a proposal, you have to generate a "rigorous" proof that the constraint indeed holds. For an example see Lemma 3,4,5 in the convex optimization problem of White and how it is related to the convex optimization problem in section 5.
4. Store the proof in the current version folder. Name it as rigorousproof.md
5. Now this is the important: you should be idle until the theorist agent reads your proof and gives a verdict. It will be in the same folder, verdict.md. If the verdict is VALID, it means your proposed proof is rigorous and makes sense. If the verdict is INVALID, it means your proposed proof has a flaw.
6. IF the verdict was VALID, implement the constraint in the convex problem.
7. IF the verdict was INVALID, go to 2 and propose a different constraint that makes mathematical sense.
 
*** THINGS TO CONSIDER ***
- A validation should take less than 10 minutes. Don't heavily optimize the parameter space (split the original problem into too many different convex optimization problems) or increase the problem parameters too much. Our job is NOT to improve the bound for now, but to improve the PROGRAM that will eventually lead to better bounds. For this problem specifically, choose a hard square region h_1 = h_2 = 0.015, p_1 = p_2 = 0.381, -q_1=q_2=0.02 to verify and see if the bound is improving.
 
- MATHEMATICAL RIGOR IS THE MOST IMPORTANT!!! Do not simply assume that you are doing an approximation that is good enough or do not propose things that may or may not be a mathematically valid lower or upper bound.

- NEVER STOP: Once the experiment loop has begun (after the initial setup), do NOT pause to ask the human if you should continue. Do NOT ask "should I keep going?" or "is this a good stopping point?". The human might be asleep, or gone from a computer and expects you to continue working indefinitely until you are manually stopped. You are autonomous. If you run out of ideas, think harder - read papers referenced in the code, re-read the in-scope files for new angles, try combining previous near-misses, try more radical architectural changes. The loop runs until the human interrupts you, period.
 
- Especially, I will not initiate a new interaction when the verdict.md is given. So when I say idle I mean don't stop the current session and wait.
\end{lstlisting}
\end{tcolorbox}

\subsection*{\texttt{Problem\_erdos.md} \textnormal{(problem description, Erd\H{o}s minimum overlap problem)}}
\begin{tcolorbox}[breakable,colback=green!4,colframe=green!55,fontupper=\small\ttfamily]
\begin{lstlisting}[breaklines=true,
breakatwhitespace=true,
columns=fullflexible,]
We want a lower bound on [ \inf_f; \sup_{x\in[-2,2]} \int_{-1}^1 f(t), g(x+t),dt, ] under the assumptions [ f,g:[-1,1]\to[0,1],\qquad g=1-f \text{ on }[-1,1]. ]

Since (g) is determined by (f) on ([-1,1]), this is really an optimization over a single function (f). The quantity inside, [ \int_{-1}^1 f(t),g(x+t),dt, ] measures the overlap between (f) and a shifted copy of its complement (g). For each shift (x\in[-2,2]), it gives the overlap at that displacement, and then the supremum over (x) picks the worst-case overlap. The infimum over (f) asks for the smallest possible such worst-case overlap.

To prove a valid lower bound (L), we need to show that every admissible (f) satisfies [ \sup_{x\in[-2,2]} \int_{-1}^1 f(t),g(x+t),dt \ge L. ]

A natural way to do this is to replace the original infinite-dimensional problem by a constrained optimization problem. The idea is to impose a collection of constraints that every genuine admissible pair ((f,g)) must satisfy, and then optimize over all objects satisfying those constraints.

This is useful because of the relaxation principle:

every admissible (f) gives a feasible point of the constrained problem,
the constrained problem is designed to be easier to analyze or solve,
and minimizing over a larger feasible set can only decrease the value.
Therefore, if the constrained optimization problem has optimum (L_{\mathrm{relax}}), then automatically [ L_{\mathrm{relax}} \le \inf_f; \sup_{x\in[-2,2]} \int_{-1}^1 f(t), g(x+t),dt. ] So (L_{\mathrm{relax}}) is a valid lower bound for the original problem.

In other words, to obtain a rigorous lower bound, we should construct a tractable optimization problem whose feasible set contains all admissible (f) and encodes necessary conditions such as [ 0\le f\le 1,\qquad 0\le g\le 1,\qquad f+g=1 \text{ on }[-1,1], ] together with any additional structural constraints we can prove. Solving that constrained problem then gives a mathematically valid lower bound for the original inf-sup quantity.

So the overall strategy is:

start from the exact inf-sup problem,
derive necessary constraints satisfied by every admissible (f),
build a constrained optimization problem from those necessary conditions,
use its optimum as a provable lower bound.
The current SOTA is 0.379005.
\end{lstlisting}
\end{tcolorbox}
 
\subsection*{\texttt{verifier.md} \textnormal{(theory agent)}}
\begin{tcolorbox}[breakable,colback=red!3,colframe=red!50,fontupper=\small\ttfamily]
\begin{lstlisting}[breaklines=true,
breakatwhitespace=true,
columns=fullflexible,]

Hi Claude, our mission is to verify proposals of a series of Autocorrelation inequalities that has been proposed. Imagine you are an expert in functional analysis, optimization, Fourier analysis, and combinatorics. Also imagine you are extra critical of other people's proposals and you want them to be crystal clear and very rigorous.
 
At each problem folder, you will have a subfolder named {versionname}. At each most recent {versionname}, if the versionname is bigger than 0, you will see a rigorousproof.md file.
 
Your job is to verify that the rigorousproof.md file is indeed rigorous and you can trust the proof.
 
I want you to first review the proof, and write a brief summary of the proof.
 
Then see if there are any counterexamples for each line of the proof. If there is a counterexample, you should report the counterexample in the same {versionname} folder, in verdict.md file.
 
If there is no issue in the proof, in verdict.md after you summarize and noting some comments, give the final verdict as VALID.
 
If you found a logical fallacy or a counterexample, give the final verdict as INVALID.
 
Also this is optional: you can explain whether the approach was effective or not, and if you have a big picture you can guide the proposal agent as a feedback too.
 
*** THINGS TO CONSIDER ***
 
- MATHEMATICAL RIGOR IS THE MOST IMPORTANT!!! Do not simply assume that you are doing an approximation that is good enough or do not propose things that may or may not be a mathematically valid lower or upper bound.
 
- Monitor every 5 minutes
 
- NEVER STOP: Once the experiment loop has begun (after the initial setup), do NOT pause to ask the human if you should continue. Do NOT ask "should I keep going?" or "is this a good stopping point?". The human might be asleep, or gone from a computer and expects you to continue working indefinitely until you are manually stopped. You are autonomous. If you run out of ideas, think harder - read papers referenced in the code, re-read the in-scope files for new angles, try combining previous near-misses, try more radical architectural changes. The loop runs until the human interrupts you, period.
\end{lstlisting}
\end{tcolorbox}
 
\newpage
\section{A comprehensive list of branch and bound for the Erd\H{o}s minimum overlap problem}
\label{app:overlap}

\begin{table}[H]
\centering
\caption{Certified lower bounds on $\Omega$ for the twelve single-cell paper rows
(A1--A5, A8, A10--A15), each verified at $(N, T, R) = (10000, 4000, 10)$
via dual certification.}
\label{tab:paper-singletons}
\begin{tabular}{l cc cc cc r}
\toprule
$h_1$ & $h_2$ & $p_1$ & $p_2$ & $q_1$ & $q_2$ & $\Omega $ \\
\midrule
0.75 & 2.00 & 0.00 & 1.00 & $-1.000$ & $1.000$ & 0.461948 \\
0.40 & 0.75 & 0.00 & 1.00 & $-1.000$ & $1.000$ & 0.409417 \\
0.20 & 0.40 & 0.00 & 1.00 & $-1.000$ & $1.000$ & 0.398605 \\
0.10 & 0.20 & 0.00 & 1.00 & $-1.000$ & $1.000$ & 0.385800 \\
0.08 & 0.10 & 0.00 & 1.00 & $-1.000$ & $1.000$ & 0.385068 \\
0.00 & 0.08 & 0.00 & 1.00 & $\phantom{-}0.050$ & $1.000$ & 0.381533 \\
0.00 & 0.08 & 0.00 & 0.25 & $-0.025$ & $0.025$ & 0.398928 \\
0.00 & 0.08 & 0.25 & 0.30 & $-0.025$ & $0.025$ & 0.389547 \\
0.00 & 0.08 & 0.30 & 0.33 & $-0.025$ & $0.025$ & 0.383319 \\
0.00 & 0.08 & 0.50 & 1.00 & $-0.025$ & $0.025$ & 0.397131 \\
0.00 & 0.08 & 0.45 & 0.50 & $-0.025$ & $0.025$ & 0.386126 \\
0.06 & 0.08 & 0.33 & 0.45 & $-0.025$ & $0.025$ & 0.390070 \\
\bottomrule
\end{tabular}
\end{table}

\begin{figure}[H]
    \centering

    \begin{subfigure}{0.48\linewidth}
        \centering
        \includegraphics[width=\linewidth]{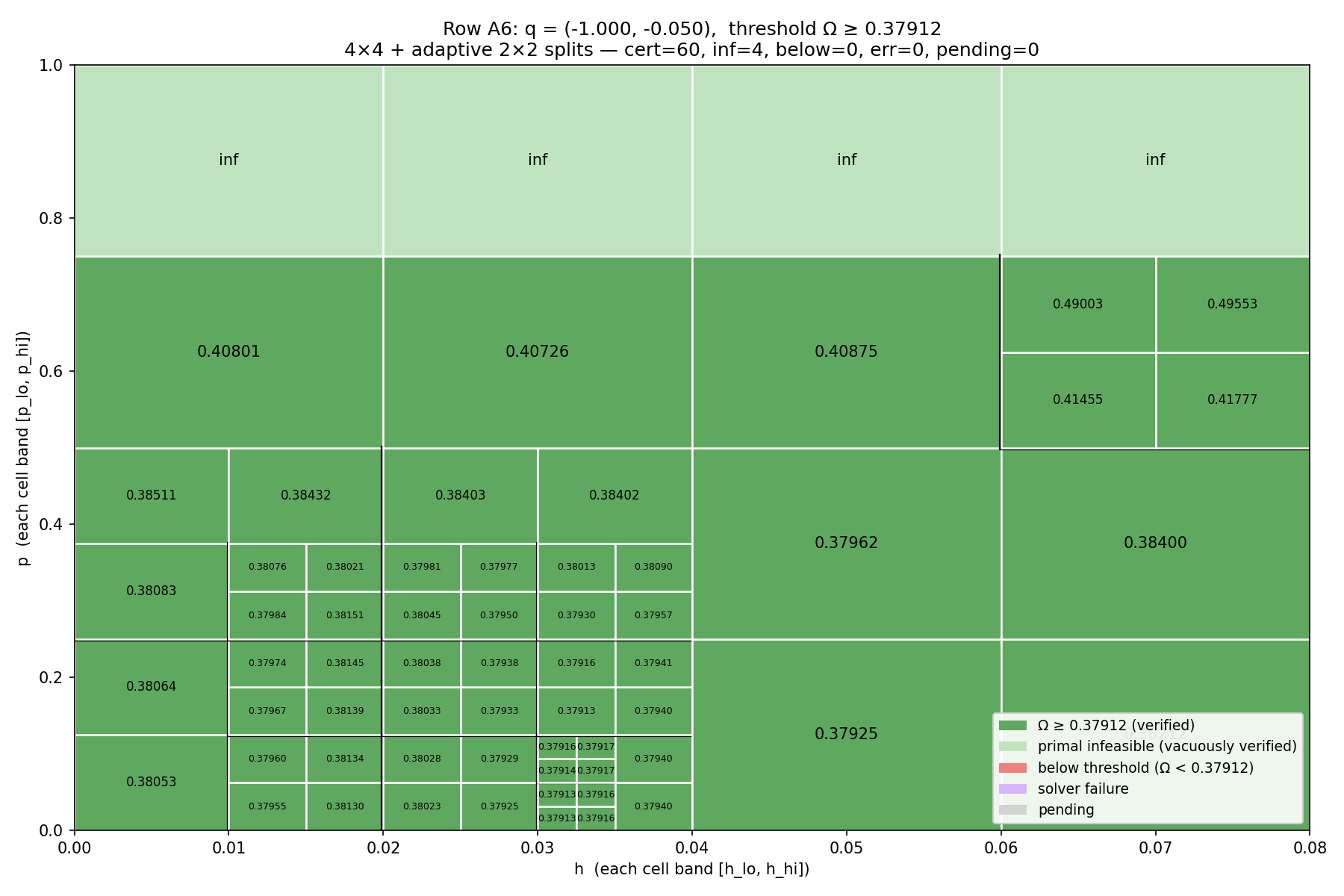}
        \caption{}
    \end{subfigure}
    \hfill
    \begin{subfigure}{0.48\linewidth}
        \centering
        \includegraphics[width=\linewidth]{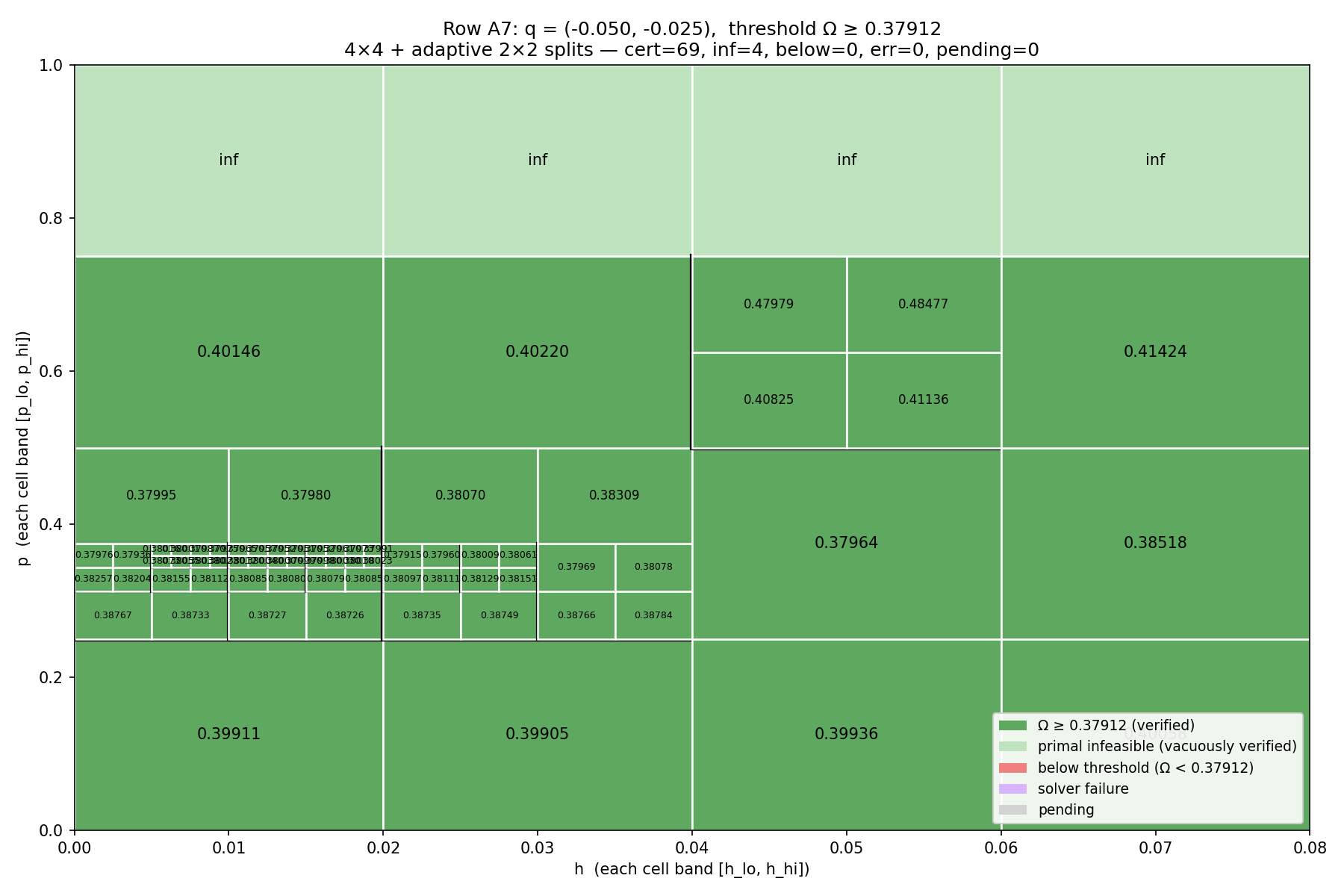}
        \caption{}
    \end{subfigure}

    \vspace{0.5em}

    \begin{subfigure}{0.48\linewidth}
        \centering
        \includegraphics[width=\linewidth]{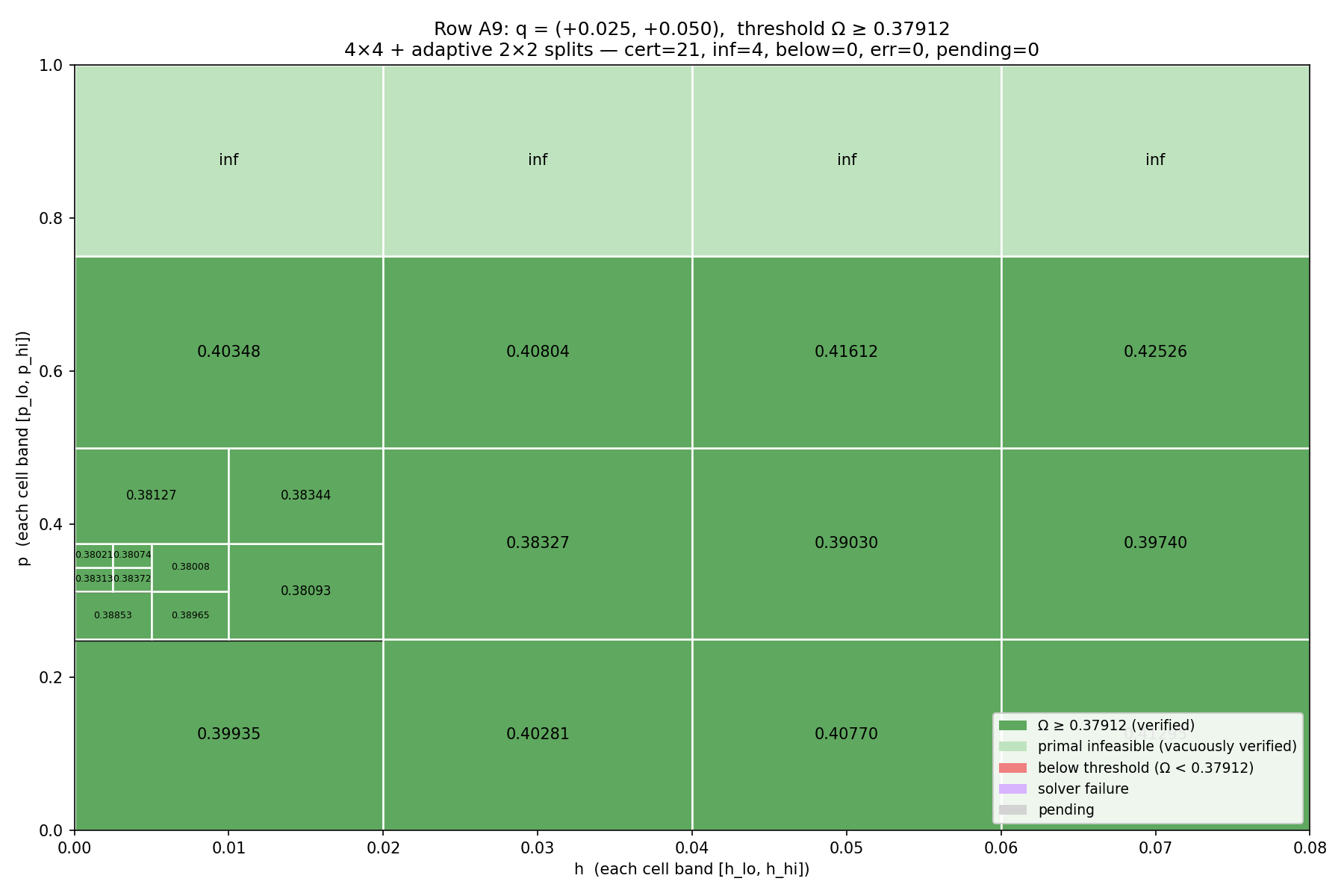}
        \caption{}
    \end{subfigure}
    \hfill
    \begin{subfigure}{0.48\linewidth}
        \centering
        \includegraphics[width=\linewidth]{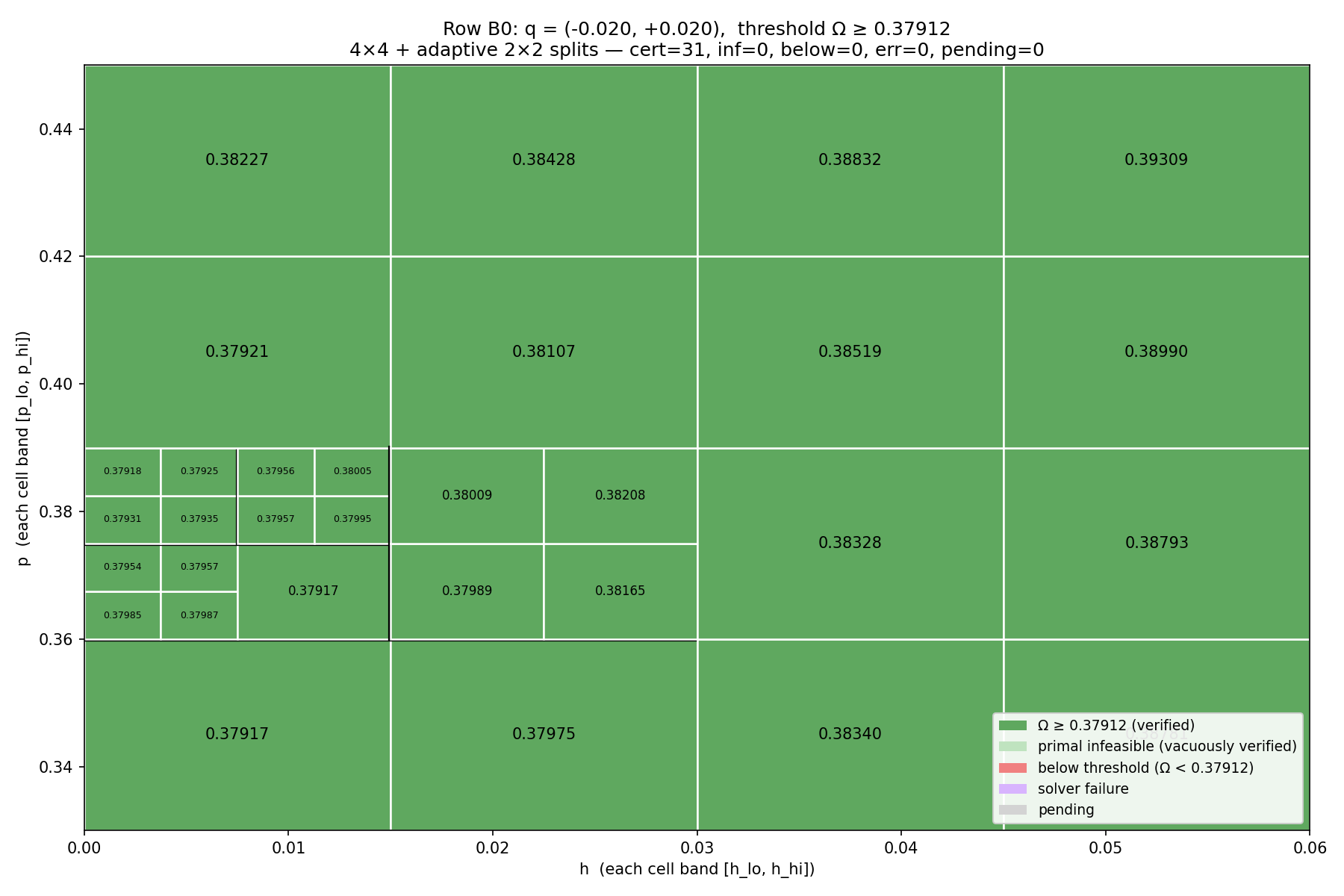}
        \caption{}
    \end{subfigure}

    \vspace{0.5em}

    \begin{subfigure}{0.48\linewidth}
        \centering
        \includegraphics[width=\linewidth]{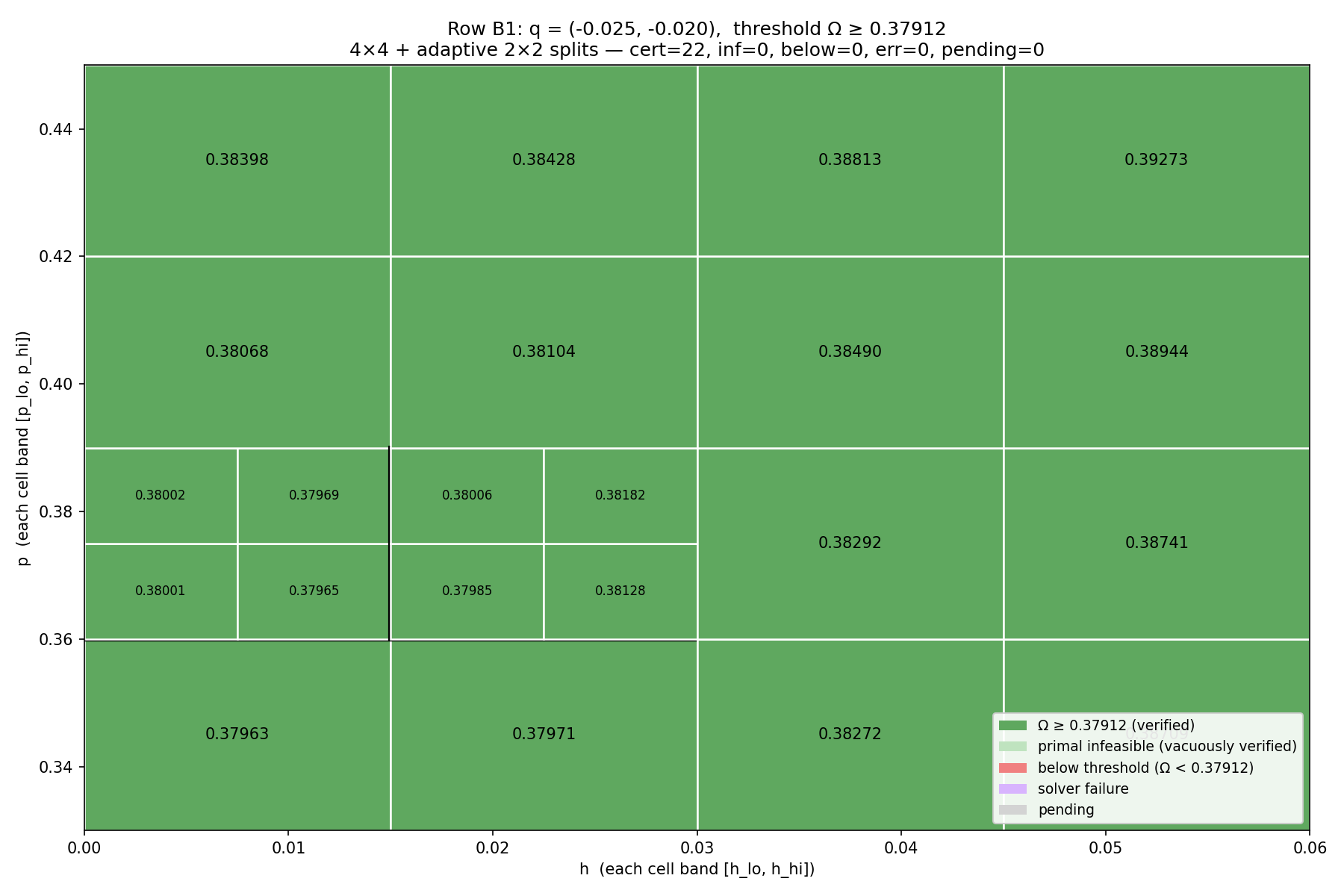}
        \caption{}
    \end{subfigure}
    \hfill
    \begin{subfigure}{0.48\linewidth}
        \centering
        \includegraphics[width=\linewidth]{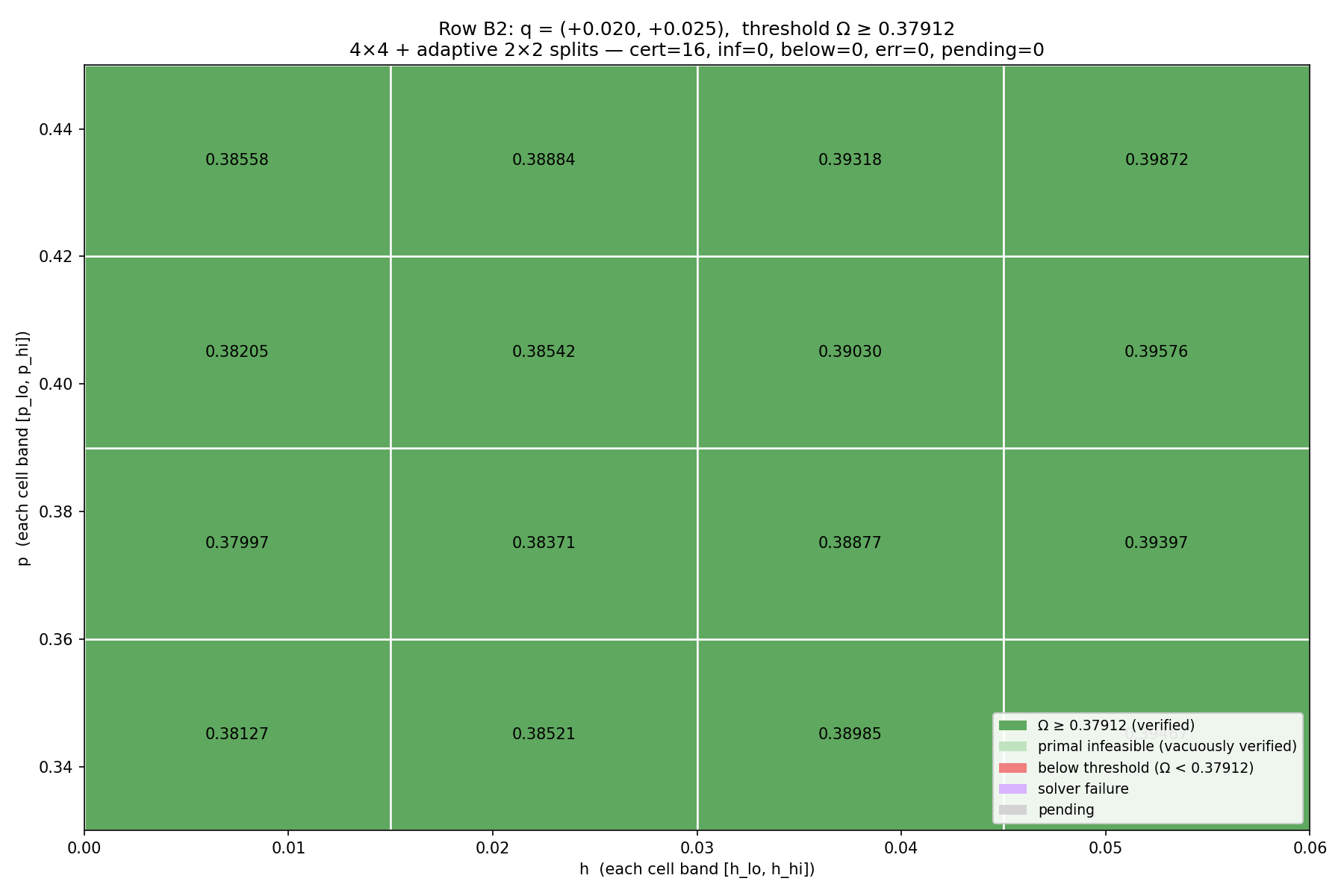}
        \caption{}
    \end{subfigure}

    \caption{Splitting $(h,p)$ for lower bound verification.}
    \label{fig:bb-tree}
\end{figure}

\newpage
\section{Interaction Between the Theory and Coding agent}
\label{AppxF}

\subsection*{\texttt{rigorousproof.md v7 [snippet selected]} \textnormal{(coding agent)}}
\begin{figure}[H]
    \centering
    \includegraphics[width=0.8\linewidth]{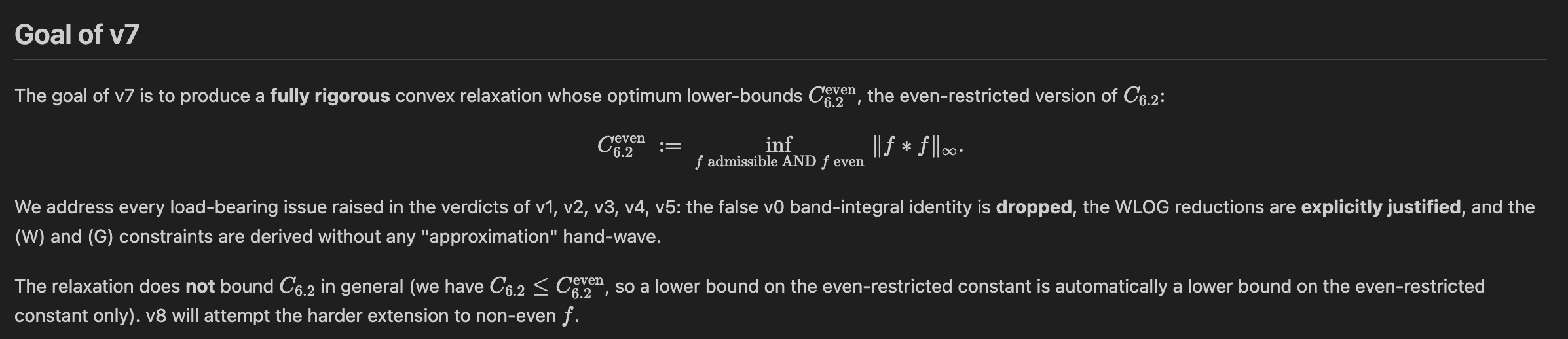}
    \caption{v7 rigorous proof.md, evidence of even function limitation}
\end{figure}

\subsection*{\texttt{verdict.md v7 [snippet selected]} \textnormal{(theory agent)}}
\begin{figure}[H]
    \centering
    \includegraphics[width=0.8\linewidth]{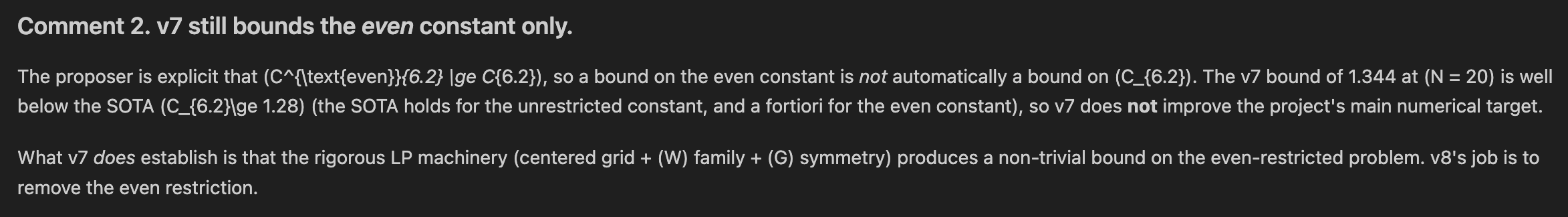}
    \caption{v7 verdict.md sampled, theory agent correction}
\end{figure}

\subsection*{\texttt{rigorousproof.md v14 [snippet selected]} \textnormal{(coding agent)}}
\begin{figure}[H]
    \centering
    \includegraphics[width=0.8\linewidth]{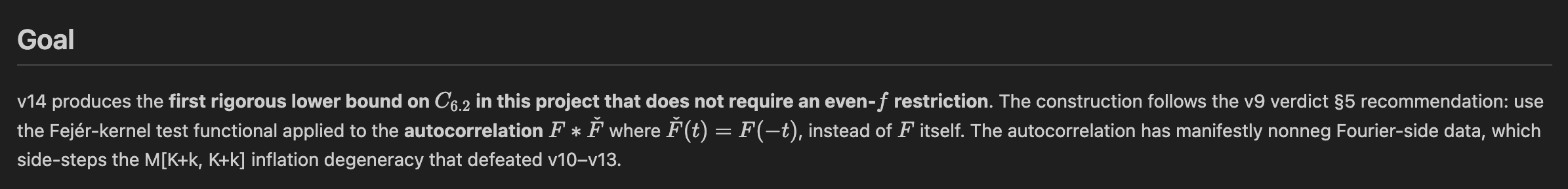}
    \caption{v14 rigorous proof.md, evidence of correction}
\end{figure}

\end{document}